%% file: main.tex
\documentclass[10pt]{article} 

\usepackage[preprint]{tmlr}

\input{math_commands.tex}

\usepackage{hyperref}
\usepackage{url}
\usepackage{lmodern}
\usepackage{microtype}
\usepackage{amsmath,amssymb,amsfonts,amsthm,mathtools}
\usepackage{bm}
\usepackage{bbm}
\usepackage{mathrsfs}
\usepackage{enumitem}
\usepackage{booktabs}
\usepackage{multirow}
\usepackage{array}
\usepackage{xcolor}
\usepackage[nameinlink,capitalise,noabbrev]{cleveref}
\usepackage{graphicx}
\usepackage{tikz-cd}
\usepackage{algorithm}
\usepackage{algpseudocode}
\usepackage{caption}
\usepackage{subcaption}
\usepackage{setspace}
\usepackage{comment}
\usepackage{float}              
\usepackage[section]{placeins}  
\usepackage[table]{xcolor}

\newtheorem{theorem}{Theorem}[section]
\newtheorem{proposition}[theorem]{Proposition}
\newtheorem{lemma}[theorem]{Lemma}

\newtheorem{definition}[theorem]{Definition}

\newcommand{\tr}{\operatorname{Tr}}
\newcommand{\rank}{\operatorname{rank}}
\newcommand{\diag}{\operatorname{diag}}

\newcommand{\ip}[2]{\left\langle #1, #2 \right\rangle}
\newcommand{\norm}[1]{\left\lVert #1 \right\rVert}
\newcommand{\abs}[1]{\left\lvert #1 \right\rvert}
\newcommand{\Had}{\odot}
\newcommand{\mc}[1]{\mathcal{#1}}

\newcommand{\Hadpow}[2]{\left(#1\right)^{\Had #2}}
\newcommand{\Hadsqrt}[1]{\Hadpow{#1}{\frac{1}{2}}}

\title{A Gauge Theory of Superposition: Toward a Sheaf-Theoretic Atlas of Neural Representations}

\author{\name Hossein Javidnia \email hossein.javidnia@dcu.ie \\
      \addr School of Computing\\
      Dublin City University, Ireland}

\def\month{02}  
\def\year{2026} 
\def\openreview{\url{https://openreview.net/forum?id=XXXX}} 

\begin{document}

\maketitle

\begin{abstract}
We develop a discrete gauge-theoretic framework for superposition in large language models (LLMs) that replaces the single-global-dictionary premise with a sheaf-theoretic atlas of local semantic charts.
Contexts are clustered into a stratified context complex; each chart carries a local feature space together with a local information-geometric metric (Fisher/Gauss--Newton) that identifies which feature interactions are predictively consequential \cite{amari2016infogeom,martens2015kfac}.
This induces a Fisher-weighted interference energy and three geometric obstructions to global interpretability:
\textbf{(O1) Local jamming} (active load exceeds Fisher bandwidth), \textbf{(O2) proxy shearing} (mismatch between geometric transport and a fixed correspondence proxy across overlaps), and \textbf{(O3) nontrivial holonomy} (path-dependent transport around loops).

We prove and empirically instantiate four results on a frozen baseline open-weight LLM (\texttt{Llama~3.2~3B Instruct}) using standard text and code corpora (WikiText-103 raw, a C4-derived English web-text subset, and \texttt{the-stack-smol}) \cite{meta2024llama32modelcard,merity2016wikitext,raffel2020t5,kocetkov2022stack}, and we further assess robustness through additional model-, layer-, seed-, null-, graph-, and transport-level ablations.
(A) A constructive spanning-tree gauge theorem shows that, after gauge fixing on a spanning tree, each chord residual equals the holonomy of its fundamental cycle, making holonomy computable and gauge-invariant.
(B) Shearing lower-bounds a data-dependent transfer mismatch energy, turning $D_{\mathrm{shear}}$ into an unavoidable failure bound.
(C) A non-vacuous certified jamming/interference bound holds on consequential atom subsets with high coverage and zero violations across random seeds and hyperparameters.
(D) Bootstrap and sample-size experiments establish estimator stability for $D_{\mathrm{shear}}$ and $D_{\mathrm{hol}}$, with improved concentration on persistent (well-conditioned) subsystems predicted by polar-factor perturbation theory \cite{higham1986polar,schonemann1966procrustes}. Additional appendix experiments show that the diagnostics transfer across multiple model families and layers, remain stable across downstream seeds, respond strongly to a random-bases null control, and are robust to moderate changes in atlas, graph, and transport settings.
Our framework supplies reproducible evaluation primitives for \emph{when and why} global interpretability breaks, separating failures due to local overload (jamming), cross-chart mismatch (shearing), and intrinsically path-dependent transport (holonomy). \\
Code: \url{https://github.com/hosseinjavidnia/gauge_superposition}
\end{abstract}

\section{Introduction}

Mechanistic interpretability often assumes that a single feature dictionary can explain a layer uniformly across contexts \cite{bereska2024mi_review,conmy2023acd}.
However, both toy models of superposition and large-scale feature-learning results suggest that this premise can fail: representations may encode more features than ambient dimension by packing sparse directions, inducing interference and context-dependent feature identities \cite{elhage2022toy,anthropic2023monosemantic,anthropic2024scaling}.
We therefore reframe the goal: interpretability should be treated as a \emph{local-to-global} construction problem, where feature descriptions are locally coherent but may exhibit measurable obstructions to being glued into a single global account.

\subsection{Contribution and paper thesis}
We model the space of contexts by a stratified simplicial complex (implemented as a graph of context clusters), and we attach to each cluster a local feature space.
Transitions between neighbouring clusters are represented by \emph{partial transports}; loop compositions define holonomy.
The central thesis is that failures of global interpretability can be decomposed into three distinct obstructions (O1--O3) that are measurable from a frozen model.

To avoid ambiguity about what is proven versus what is measured, we isolate four results (A--D) that connect theory to experiments:

\begin{itemize}[leftmargin=2em]
\item \textbf{(A) Spanning-tree gauge identity.} After a constructive gauge fixing on a spanning tree, tree-edge defects vanish and each chord defect equals the holonomy of its fundamental cycle (Theorem~\ref{thm:A}).
This makes holonomy computable and interpretable, rather than an abstract loop product.

\item \textbf{(B) Shearing implies unavoidable mismatch.} The shearing score $D_{\mathrm{shear}}(u,v)$ lower-bounds a concrete transfer mismatch energy on overlap data (Theorem~\ref{thm:B}).
This turns shearing into an obstruction with operational meaning.

\item \textbf{(C) Certified jamming forces interference.} A non-vacuous lower bound certifies forced interference on a consequential subset of atoms (Theorem~\ref{thm:C}).
We evaluate not merely correctness, but also non-vacuity (coverage), and robustness across seeds and hyperparameters.

\item \textbf{(D) Estimator stability.} Bootstrap and sample-size curves establish that $D_{\mathrm{shear}}$ and $D_{\mathrm{hol}}$ are stable measurements under finite overlap sampling, especially on well-conditioned persistent subsystems (Section~\ref{sec:D}).
\end{itemize}

To test whether these diagnostics reflect structural properties of learned representations rather than artefacts of one checkpoint or one implementation setting, the appendix reports a broader suite of robustness experiments. In particular, we report additional model-generalisation experiments (\texttt{Qwen~2.5~3B Instruct} \cite{qwen2024qwen25,qwen2024qwen25_3b_instruct} and \texttt{Gemma~2~2B~IT} \cite{gemmateam2024gemma2,google2024gemma2_2b_it}), a layerwise sweep within the baseline Llama model, downstream seed-reproducibility tests, a random-bases null control, and robustness checks for atlas, graph, and transport hyperparameters. These experiments are intended to distinguish structural properties of the learned representation from artefacts of a single model or configuration.

\subsection{Significance and impact}
Most mechanistic interpretability work implicitly assumes that a single, globally consistent feature dictionary can describe a layer across diverse contexts.
Our results show that this premise can fail in structured, measurable ways even for a frozen model, and we propose a concrete alternative: interpretability as an \emph{atlas construction problem} with explicit local-to-global obstructions.
The central impact is methodological: we turn context dependence from an informal caveat into a set of \emph{computable, falsifiable diagnostics} that decompose global failure into three distinct mechanisms---local overload (jamming), cross-chart incompatibility (proxy shearing), and intrinsically path-dependent transport (holonomy).

Concretely, Result~A provides a constructive gauge computation that reduces holonomy from an abstract loop statistic to a directly checkable quantity: after spanning-tree gauge fixing, chord residuals coincide with fundamental-cycle holonomy, and the identity checks hold to numerical precision.
Result~B connects proxy disagreement to operational consequences by lower-bounding a realised transfer-mismatch energy, making large shearing edges interpretable as unavoidable cross-context mismatch \emph{with respect to a fixed correspondence proxy}.
Result~C shows that interference need not remain qualitative: we obtain \emph{non-vacuous} certified lower bounds on consequential subsets with high coverage and zero violations across random seeds and hyperparameters.
Finally, Result~D establishes that the proposed measurements are not fragile artefacts of overlap sampling: bootstrap and sample-size experiments show concentration with increasing overlap size and improved stability under conditioning-based persistence filtering.

Taken together, the framework supplies evaluation primitives for \emph{when} and \emph{why} global interpretability breaks, enabling atlas-level summaries that are reproducible, stability-tested, and comparable across subsystems, rather than anecdotal case studies. The appendix further shows that these diagnostics are not confined to a single checkpoint or one hand-tuned configuration: they transfer across multiple model families, vary systematically across depth, remain stable across downstream seeds, respond strongly to a random-bases null control, and are robust to moderate changes in atlas, graph, and transport construction.

\subsection{Related work}
Superposition toy models and monosemantic feature extraction motivate our focus on interference and overcomplete features \cite{elhage2022toy,anthropic2023monosemantic,anthropic2024scaling}.
Our use of Fisher/Gauss--Newton geometry follows classical information geometry and second-order optimisation literature \cite{amari2016infogeom,martens2015kfac}.
The atlas perspective is naturally expressed using cellular sheaves, which provide a principled language for local-to-global structure over combinatorial bases \cite{curry2014thesis,hansen2019spectral}. Neural message-passing models have recently incorporated sheaf-theoretic structure explicitly, providing complementary perspectives on local-to-global consistency and transport constraints \cite{bodnar2022neural_sheaf_diffusion}. In particular, cooperative sheaf constructions study how multiple local predictors can be coupled via sheaf constraints, offering a modern ML analogue of atlas-style consistency conditions \cite{ribeiro2025cooperative_sheaf_nn}.
Our gauge/holonomy viewpoint is aligned with group synchronisation on graphs \cite{singer2011angular}, but differs in that the underlying objects are derived from learned transports and correspondence proxies rather than direct measurements of group elements.
Recent work has proposed \emph{representation holonomy} as a gauge-invariant statistic for path dependence in learned representations, with estimators based on (whitened) local Procrustes alignment around loops in input space \cite{sevetlidis2026holonomy}.
Our setting differs in both object and goal: we build an \emph{atlas over context clusters} (a discrete context complex) and study transports/defects on the resulting chart graph rather than input-space neighborhoods.
Moreover, we connect holonomy to a constructive \emph{spanning-tree gauge computation} that makes fundamental-cycle holonomy directly computable from chord defects (Theorem~\ref{thm:A}), and we complement holonomy with two additional, operational obstructions---proxy shearing (Theorem~\ref{thm:B}) and certified jamming/interference (Theorem~\ref{thm:C}). 

\paragraph{Scope.}
\cite{sevetlidis2026holonomy} introduces holonomy as a representation-level diagnostic on input-space loops, whereas we use holonomy as one component of a broader atlas-based interpretability objective on a context-cluster graph.
Accordingly, our shearing and jamming analyses are intentionally proxy- and model-class--specific (see Limitations), and our contribution is the combination of (i) an explicit chart-graph construction, (ii) a gauge-fixing computation that reduces holonomy to chord defects (Theorem~\ref{thm:A}), and (iii) complementary certified and stable diagnostics for local interference and cross-chart mismatch (Theorems~\ref{thm:B}--\ref{thm:C}, Result~D).

\FloatBarrier

\section{Setup: context complex, local charts, transports}
\label{sec:setup}

\subsection{Activations and clustering}
Fix a frozen model and layer $\ell$ with activation space $\R^d$.
From a token dataset $\mc X$, extract token-level activations $\{x_t\}_{t=1}^N\subset\R^d$.
Cluster the activations into $C$ clusters $\{\mc D_c\}_{c=1}^C$ \cite{macqueen1967kmeans}.
Clusters act as \emph{local charts}.

\subsection{Stratified context complex}
\begin{definition}[Stratified context complex]
Let $\mc K$ be a finite simplicial complex (in practice, we use its $1$-skeleton $G=(V,E)$) whose vertices index context clusters.
Assume a stratification $\mc K=\bigsqcup_{s\in\mc S}\mc K_s$ where each stratum is approximately homogeneous with respect to a context class (e.g.\ prose vs code).
\end{definition}

Stratification is not required for the definitions below, but is used to interpret which transitions induce large shearing or holonomy.

\subsection{Local feature spaces and frames}
\begin{definition}[Local frame]
Let $U\subseteq\R^d$ be finite-dimensional.
A set $\{e_i\}_{i=1}^m\subset U$ is a \emph{frame} if it spans $U$ \cite{christensen2016frames}.
It is \emph{overcomplete} if $m>\dim U$.
\end{definition}

In experiments we use two chart representations:
(i) an orthonormal $k$-dimensional PCA chart basis (for transports/holonomy), and
(ii) a local overcomplete dictionary learned by sparse coding (for jamming/interference).

\subsection{Rectangular partial transports}
\begin{definition}[Rectangular partial transport]
For each oriented edge $u\to v$ in the chart graph, define
\[
T_{vu}:F_u\to F_v,\qquad T_{vu}\in\R^{m_v\times m_u}.
\]
\end{definition}

Rank variation accommodates partial persistence, merging, splitting, and birth/death of features across contexts.

\FloatBarrier

\section{Local information geometry and Fisher-weighted interference}
\label{sec:local-geom}

\subsection{Local Fisher/Gauss--Newton metric}
Let $p_\theta(y\mid x)$ be the model output distribution.
In local coordinates $z=z^{(c)}(x)$ define
\begin{equation}
\label{eq:fisher}
G^{(c)} := \E_{x\sim \mc D_c}\big[g_z(x)\,g_z(x)^\top\big]\in\R^{m_c\times m_c},
\qquad
g_z(x) := D_c\,g_x(x),
\qquad
g_x(x):=\nabla_{x_\ell}\ell(x)\in\R^d,
\end{equation}
where $\ell(x)$ is the next-token negative log-likelihood of the frozen model, $x_\ell$ denotes the layer-$\ell$ activation vector, and $D_c\in\R^{m_c\times d}$ is the learned per-chart \emph{encoder} mapping activations to code space (we write codes as $z=D_c x_\ell$). Equivalently, if $E_c\in\R^{d\times m_c}$ denotes a decoder dictionary with column atoms, then $D_c$ is the corresponding encoder used for code-space gradients.
In experiments we compute $g_x$ by backpropagating $\ell$, map to code-space gradients $g_z=D_c g_x$, and estimate $G^{(c)}$ by the empirical outer product
\[
\widehat G^{(c)}=\frac{1}{n}\sum_{i=1}^n g_{z,i}\,g_{z,i}^\top.
\]

\begin{proposition}[PSD]
\label{prop:psd}
$G^{(c)}$ is symmetric positive semidefinite.
\end{proposition}

\begin{proof}
$G^{(c)}=\E[g_z g_z^\top]$ with $g_z = D_c\nabla_{x_\ell}\ell(x)$, hence
$a^\top G^{(c)}a=\E\!\big[(a^\top g_z)^2\big]\ge 0$.
\end{proof}

\subsection{Harm matrix}
Raw off-diagonal Fisher entries can be negative; using them directly would reward certain overlaps.
We instead construct nonnegative interaction weights.

\begin{definition}[Harm matrix]
Fix $\tau>0$.
Define the diagonally normalised Fisher
\[
\widetilde G^{(c)} :=
\big(\diag(G^{(c)})+\tau I\big)^{-1/2}\,
G^{(c)}\,
\big(\diag(G^{(c)})+\tau I\big)^{-1/2}.
\]
Define $W^{(c)}\in\R_+^{m_c\times m_c}$ with zero diagonal by
\[
W^{(c)}_{ij}:=
\begin{cases}
\eta\!\left(\abs{\widetilde G^{(c)}_{ij}}\right), & i\neq j,\\
0, & i=j,
\end{cases}
\]
with monotone $\eta$ (canonical: $\eta(t)=t$).
\end{definition}

\subsection{Frame geometry and interference energy}
Let $E_c\in\R^{d\times m_c}$ be a decoder/dictionary matrix with columns $e^{(c)}_i$.
Let $\widehat E_c=E_c(D_c^{(E)})^{-1/2}$ where $D_c^{(E)}=\diag(\norm{e_1}^2,\dots,\norm{e_{m_c}}^2)$.
Define $K^{(c)}=\widehat E_c^\top \widehat E_c$.

\begin{definition}[Local Fisher-weighted interference energy]
\label{def:energy}
\[
\mc E(c) := \sum_{i\neq j} W^{(c)}_{ij}\big(K^{(c)}_{ij}\big)^2
\;=\;
\Big\|\Hadsqrt{W^{(c)}}\Had\big(K^{(c)}-I\big)\Big\|_F^2.
\]
\end{definition}

\begin{proposition}[Nonnegativity and zero set]
\label{prop:nonneg}
$\mc E(c)\ge 0$.
Moreover, $\mc E(c)=0$ iff $K^{(c)}_{ij}=0$ for all $i\neq j$ with $W^{(c)}_{ij}>0$.
\end{proposition}

\begin{proof}
Each summand is nonnegative; the sum is zero iff each summand is zero.
\end{proof}

\subsection{Effective rank and the jamming index}
\begin{definition}[Effective rank]
For PSD $G\neq 0$,
\[
R_{\mathrm{eff}}(G):=\frac{\tr(G)^2}{\tr(G^2)}.
\]
\end{definition}

\begin{proposition}[Basic properties]
\label{prop:reff}
$R_{\mathrm{eff}}(\alpha G)=R_{\mathrm{eff}}(G)$ for $\alpha>0$, and $1\le R_{\mathrm{eff}}(G)\le\rank(G)$.
\end{proposition}

\begin{definition}[Participation-ratio active count]
For codes $z\in\R^{m_c}$ define $k_{\mathrm{PR}}(z)=\norm{z}_1^2/\norm{z}_2^2$ (with $k_{\mathrm{PR}}(0)=0$) and
\[
k_{\mathrm{active}}(c):=\E[k_{\mathrm{PR}}(z^{(c)}(x))].
\]
\end{definition}

\begin{definition}[Local jamming index]
\label{def:jamming}
\[
J(c):=\frac{k_{\mathrm{active}}(c)}{R_{\mathrm{eff}}(G^{(c)})}.
\]
\end{definition}

\subsection{Welch-type baseline lower bound}
\begin{proposition}[Weighted Welch-type lower bound]
\label{prop:welch}
Let $A$ index $k$ unit vectors lying in an $r$-dimensional subspace.
If $(W_A)_{ij}\ge w_{\min}>0$ for all $i\neq j$ in $A$, then
\[
\mc E_A(c):=\sum_{i\neq j\in A}W^{(c)}_{ij}\big(K^{(c)}_{ij}\big)^2 \ge w_{\min}\left(\frac{k^2}{r}-k\right).
\]
\end{proposition}

\begin{proof}
Standard Welch bound \cite{welch1974crosscorr}: $\tr(G_A^2)\ge (\tr G_A)^2/\rank(G_A)\ge k^2/r$ implies $\sum_{i\neq j}\ip{e_i}{e_j}^2\ge k^2/r-k$. Multiply by $w_{\min}$.
\end{proof}

This proposition is mathematically correct but often practically weak because $w_{\min}$ across all pairs can be very small.
Theorem~\ref{thm:C} resolves this by certifying a consequential subset with a nontrivial uniform weight floor.

\FloatBarrier

\section{Proxy shearing and Theorem B}
\label{sec:shearing}

\subsection{Geometric transport and correspondence proxy}
Let $(u,v)$ be a graph edge with overlap samples $X_{uv}$.
Work on a shared $k$-dimensional support (e.g.\ chart PCA subspace).
Let $Z_u,Z_v\in\R^{k\times n}$ be overlap coordinates.
Estimate a transport $\widehat T_{vu}$ and define the orthogonal polar factor $Q_{vu}=\operatorname{polar}(\widehat T_{vu})\in O(k)$ \cite{higham1986polar,schonemann1966procrustes}.

Let $\widehat P_{vu}\in O(k)$ be a fixed, reproducible correspondence proxy on the same support.
In our experiments, $\widehat P_{vu}$ is constructed as a Procrustes-style correspondence proxy from the fitted PCA chart bases (Section~\ref{sec:proxy-def}), and all reported $D_{\mathrm{shear}}$, $g_{vu}$, and holonomy values use this choice.
The shearing score measures disagreement between the geometry-derived transport and the correspondence proxy.

\paragraph{Concrete definition of $\widehat P_{vu}$ (default proxy).}\label{sec:proxy-def}
To make $D_{\mathrm{shear}}$ reproducible, we fix $\widehat P_{vu}$ as a Procrustes-style correspondence proxy derived from the local PCA chart bases.
Let $B_u,B_v\in\R^{d\times k}$ denote the orthonormal PCA bases fitted per chart (Algorithm~\ref{alg:atlas}).
Define the $k\times k$ basis overlap matrix
\[
S_{vu}:=B_v^\top B_u.
\]
We then set the proxy to be the orthogonal polar factor of $S_{vu}$:
\[
\widehat P_{vu}:=\operatorname{polar}(S_{vu})\in O(k),
\]
i.e.\ if $S_{vu}=U\Sigma V^\top$ is an SVD, then $\operatorname{polar}(S_{vu})=UV^\top$ \cite{higham1986polar,schonemann1966procrustes}.
This proxy is deterministic given the fitted chart bases and does not require an additional matching procedure on overlap statistics.

\begin{definition}[Normalised shearing]
\label{def:shear}
\[
D_{\mathrm{shear}}(u,v):=\frac{\|Q_{vu}-\widehat P_{vu}\|_F}{2\sqrt{k}}\in[0,1].
\]
\end{definition}

\subsection{Theorem B: shearing lower-bounds transfer mismatch}
Let $z_u$ denote overlap codes in chart $u$, with covariance $\Sigma_u=\E[z_u z_u^\top]\succeq 0$.
Define the transfer mismatch energy
\[
\Delta_{uv}:=\E\big\|(Q_{vu}-\widehat P_{vu})z_u\big\|_2^2.
\]

\begin{theorem}[Shearing lower-bounds transfer mismatch]
\label{thm:B}
Let $A_{vu}:=Q_{vu}-\widehat P_{vu}$ and $\Sigma_u\succeq 0$.
Then
\[
\Delta_{uv}=\tr(A_{vu}\Sigma_u A_{vu}^\top)
\ \ge\
\lambda_{\min}(\Sigma_u)\|A_{vu}\|_F^2
=
4k\,\lambda_{\min}(\Sigma_u)\,D_{\mathrm{shear}}(u,v)^2.
\]
\end{theorem}

\begin{proof}
Since $\Sigma_u\succeq \lambda_{\min}(\Sigma_u)I$ on the shared support,
$\tr(A\Sigma_u A^\top)\ge \lambda_{\min}(\Sigma_u)\tr(AA^\top)=\lambda_{\min}(\Sigma_u)\|A\|_F^2$.
Substitute $\|A\|_F=2\sqrt{k}D_{\mathrm{shear}}$.
\end{proof}

All ``mismatch'' statements here are with respect to the fixed proxy $\widehat P_{vu}$ defined above (not an absolute notion of semantics). The empirical section will report both $\widehat\Delta_{uv}$ and the lower bound $\widehat{\mathrm{LB}}_{uv}:=\lambda_{\min}(\widehat\Sigma_u)\|A_{vu}\|_F^2$, and will quantify the \emph{slack} $\widehat\Delta_{uv}/\widehat{\mathrm{LB}}_{uv}$.
Large slack indicates conservativeness caused by small $\lambda_{\min}(\widehat\Sigma_u)$ (anisotropic overlap excitation), not a failure of the theorem.

\FloatBarrier

\section{Gauge structure, holonomy, and Theorem A}
\label{sec:holonomy}

\subsection{Edge defects, gauge action, holonomy}
Assume a connected graph $G=(V,E)$ and orthogonal edge defects $g_{vu}\in O(k)$ on oriented edges.
A gauge is an assignment $U_c\in O(k)$ at each vertex, acting by
\[
g_{vu}\mapsto g'_{vu}:=U_v g_{vu}U_u^\top.
\]
For a directed loop $\gamma:c_0\to \cdots \to c_L=c_0$ define holonomy
\[
h_\gamma:=g_{c_0c_{L-1}}\cdots g_{c_1c_0}\in O(k).
\]
Under gauge, $h_\gamma$ transforms by conjugation at $c_0$, hence $\|h_\gamma-I\|_F$ is gauge invariant.
We use the normalised defect
\[
D_{\mathrm{hol}}(\gamma):=\frac{\|h_\gamma-I\|_F}{\sqrt{2k}}\in[0,1].
\]

\subsection{Theorem A: spanning-tree gauge identity}
\begin{theorem}[Spanning-tree gauge and fundamental-cycle holonomy]
\label{thm:A}
Let $G=(V,E)$ be connected and $T\subseteq G$ a spanning tree rooted at $r$ with parent map $\pi(v)$.
Define
\[
U_r:=I,\qquad U_v:=U_{\pi(v)}\,g_{\pi(v),v}\quad(v\neq r).
\]
Then:
\begin{enumerate}[label=(\roman*), leftmargin=2em]
\item For each tree edge $\pi(v)\to v$, $g'_{v,\pi(v)}=I$.
\item For each chord $e=\{u,v\}\notin T$, the gauged chord defect equals the holonomy of the fundamental cycle $\gamma_e$:
\[
g'_{vu}=h_{\gamma_e}.
\]
\end{enumerate}
\end{theorem}

\begin{proof}
Tree edges: $g'_{v,\pi(v)}=(U_{\pi(v)}g_{\pi(v),v})g_{v,\pi(v)}U_{\pi(v)}^\top=U_{\pi(v)}I U_{\pi(v)}^\top=I$ since $g_{\pi(v),v}=g_{v,\pi(v)}^{-1}$.

Chord edges: write $g'_{vu}=U_v g_{vu}U_u^\top$ and expand $U_v$ and $U_u$ as products along tree paths from $r$.
Common path factors cancel (telescoping), leaving the ordered product along the unique tree path from $u$ to $v$ and the chord $v\to u$, which is precisely the loop product defining $h_{\gamma_e}$.
\end{proof}

\subsection{Holonomy obstruction criterion}
\begin{theorem}[Holonomy obstruction]
\label{thm:hol-obs}
If there exists a gauge such that $g'_{vu}=I$ for all edges in a connected subgraph, then $h_\gamma=I$ for every loop $\gamma$ in that subgraph.
Equivalently, if some loop has $h_\gamma\neq I$, then no globally connection-compatible trivialisation exists on that subsystem.
\end{theorem}

\begin{proof}
If every edge defect is identity in a gauge, every loop product is a product of identities.
\end{proof}

Theorem~\ref{thm:A} supplies a computational strategy: choose a spanning tree in the LCC, gauge-fix it, then compute holonomy only on chords (each chord generates one fundamental cycle).
The experiments will show that (i) tree residuals are numerically near zero, and (ii) chord residuals match holonomy to numerical precision, which is the quantitative validation of Theorem~\ref{thm:A} in an LLM setting.

\FloatBarrier

\section{Certified jamming/interference: Theorem C}
\label{sec:C}

\subsection{Projected setting and consequential subsets}
Fix a chart $c$.
Let $D_c\in\R^{m\times d}$ be a dictionary with unit atoms $\widehat D_{c,i}$.
Let $W^{(c)}$ be the harm matrix in code space.
Let $r=\lceil R_{\mathrm{eff}}(G^{(c)})\rceil$ and let $B_r\in\R^{d\times r}$ be the top-$r$ PCA basis of chart activations.
Project atoms to $\R^r$ by
\[
a_i:=\widehat D_{c,i}B_r,\qquad \widehat a_i:=a_i/\|a_i\|_2,
\]
and define projected Gram $K^{(r)}_{ij}:=\ip{\widehat a_i}{\widehat a_j}$ with $K^{(r)}_{ii}=0$.

\begin{definition}[$\tau$-consequential subset]
A subset $A\subseteq\{1,\dots,m\}$ is $\tau_\star$-consequential if $W^{(c)}_{ij}\ge \tau_\star$ for all distinct $i,j\in A$.
\end{definition}

\subsection{Theorem C}
\begin{theorem}[Certified interference on a consequential subset]
\label{thm:C}
Let $A$ be $\tau_\star$-consequential and $k:=|A|$. Then
\[
\mc E^{(r)}_{A}(c):=\sum_{i\neq j\in A}W^{(c)}_{ij}\big(K^{(r)}_{ij}\big)^2
\ \ge\
\widehat{\mathrm{LB}}(c):=\tau_\star\left(\frac{k^2}{r}-k\right)_+.
\]
\end{theorem}

\begin{proof}
The vectors $\widehat a_i$ lie in $\R^r$ and are unit norm.
Let $G_A$ be their Gram matrix. Then $\rank(G_A)\le r$ and $\tr(G_A)=k$.
By Cauchy--Schwarz on eigenvalues, $\tr(G_A^2)\ge \tr(G_A)^2/\rank(G_A)\ge k^2/r$.
Hence $\sum_{i\neq j\in A}\ip{\widehat a_i}{\widehat a_j}^2=\tr(G_A^2)-k\ge k^2/r-k$.
Since $W^{(c)}_{ij}\ge\tau_\star$ on $A$, the stated bound follows.
\end{proof}

The experimental section will report (i) coverage: how many charts achieve $\widehat{\mathrm{LB}}(c)>0$, (ii) correctness: slack $\mc E_A^{(r)}(c)/\widehat{\mathrm{LB}}(c)\ge 1$ with zero violations, and (iii) robustness across seeds and $(m,\alpha)$.

\FloatBarrier

\section{Stability and concentration: Result D}
\label{sec:D}

\subsection{Why conditioning matters: polar factor perturbation}
\begin{lemma}[Polar factor stability]
\label{lem:polar}
Let $T$ be nonsingular with $\sigma_{\min}(T)\ge s>0$ and $\|\Delta T\|_2\le s/2$.
Then
\[
\|\operatorname{polar}(T+\Delta T)-\operatorname{polar}(T)\|_F
\le C\,\frac{\|\Delta T\|_F}{s}
\]
for a universal constant $C$ \cite{higham1986polar}.
\end{lemma}

\subsection{Loop accumulation bound}
\begin{lemma}[Loop product telescoping bound]
\label{lem:loop}
For $g_i,\tilde g_i\in O(k)$,
\[
\left\|\prod_{i=1}^L g_i - \prod_{i=1}^L \tilde g_i\right\|_F
\le \sum_{i=1}^L \|g_i-\tilde g_i\|_F.
\]
\end{lemma}

\begin{proof}
Expand the difference as a telescoping sum. Orthogonal left/right multiplication preserves Frobenius norm.
\end{proof}

Lemmas~\ref{lem:polar} and \ref{lem:loop} predict two testable behaviours:
(i) edge diagnostics based on polar factors are less variable when $\sigma_{\min}(T_{vu})$ is bounded away from $0$, and
(ii) loop diagnostics are stable when edge diagnostics are stable, with variability accumulating with loop length.
The experiments quantify these predictions with global bootstrap summaries and within-edge/within-loop sample-size curves.

\FloatBarrier

\section{Atlas Method (pipeline)}
\label{sec:atlas}

This section states the pipeline used in experiments and clarifies which stage produces which reported numbers.
It is included because results A--D depend on multiple intermediate objects (clusters, overlaps, transports, proxies, loops), and reproducibility requires an explicit map from objects to measurements.

\subsection{Pipeline summary}
\begin{algorithm}[H]
\caption{Atlas Method (post-hoc cartography)}
\label{alg:atlas}
\begin{algorithmic}[1]
\Require Frozen model $\theta$, layer $\ell$, dataset $\mc X$, clusters $C$, support dimension $k$
\Ensure Graph $G$, edge defects $g_{vu}$, shearing $D_{\mathrm{shear}}$, holonomy $D_{\mathrm{hol}}$, jamming diagnostics
\State Extract token-level activations $\{x_t\}_{t=1}^N\subset\R^d$ at layer $\ell$ (optionally subsampling tokens for storage/compute).
\State Cluster activations into $C$ charts; build centroid $k$-NN graph $G=(V,E)$.
\State For each chart $c$: fit PCA basis $B_c\in\R^{d\times k}$ \cite{jolliffe2002pca}.
\State For jamming/interference (Result C): collect per-chart samples $(x_\ell,g_x)$ with $g_x=\nabla_{x_\ell}\ell(x)$ (next-token NLL), fit an overcomplete dictionary $D_c$ and sparse codes $z$, and form $\widehat G^{(c)}=\frac{1}{n}\sum_i (D_c g_{x,i})(D_c g_{x,i})^\top$.
\State For each edge $(u,v)$: build overlap set $X_{uv}$; compute overlap coordinates $Z_u=B_u^\top X_{uv}$, $Z_v=B_v^\top X_{uv}$.
\State Estimate $\widehat T_{vu}$ by ridge regression \cite{hoerl1970ridge}: $\widehat T_{vu}:= Z_v Z_u^\top (Z_u Z_u^\top+\lambda I)^{-1}$; compute $Q_{vu}:=\operatorname{polar}(\widehat T_{vu})$.
\State Construct the correspondence proxy $\widehat P_{vu}:=\operatorname{polar}(B_v^\top B_u)$ on the shared support (Section~\ref{sec:proxy-def}), and define the edge defect $g_{vu}=\widehat P_{vu}^\top Q_{vu}$.
\State Compute $D_{\mathrm{shear}}(u,v)$.
\State Choose spanning tree in LCC; compute chord holonomies and $D_{\mathrm{hol}}$ via Theorem~\ref{thm:A}.
\State (Optional) Filter edges by $\sigma_{\min}(\widehat T_{vu})\ge s_{\min}$ to form persistent subsystem and repeat steps 6--7.
\end{algorithmic}
\end{algorithm}

\subsection{Interpretation of the persistence filter}
Filtering edges by $\sigma_{\min}(\widehat T_{vu})\ge s_{\min}$ operationalises the conditioning requirement suggested by Lemma~\ref{lem:polar}.
This filter changes \emph{both} numerical stability and the graph topology: it can reduce the number of cycles available for holonomy computation.
The experiments report this tradeoff explicitly in Table~\ref{tab:persist}.

\FloatBarrier

\section{Experiments}
\label{sec:experiments}

\subsection{Experimental setup}
We used \texttt{meta-llama/Llama-3.2-3B-Instruct} and extracted activations at layer $\ell=16$ with $d=3072$ \cite{meta2024llama32modelcard}.
We used three data sources to obtain heterogeneous contexts: WikiText-103 (raw) for prose, a C4-derived English web-text subset (from \texttt{PrimeIntellect/c4-tiny} when available, otherwise streamed from \texttt{allenai/c4}), and code from \texttt{bigcode/the-stack-smol} \cite{merity2016wikitext,raffel2020t5,kocetkov2022stack}.
We extracted and stored $N\approx 2\times 10^5$ token activations \emph{after subsampling tokens} (every 8 tokens per sequence by default), clustered these activations into $C=128$ clusters, and built a centroid $k$-NN graph.
For transports and holonomy we used a shared chart support dimension $k=32$.

Overlap sets $X_{uv}$ were constructed by a Voronoi rule: an activation belongs to $X_{uv}$ if its nearest and second-nearest centroids are $\{u,v\}$.
In implementation, we compute nearest and second-nearest centroids from the fitted centroid set for each activation and form overlaps accordingly.
This concentrates overlap data near the boundary between clusters, which empirically improves transport fit stability relative to using all points in both clusters.

The main text focuses on the core empirical instantiations of Results~A--D on the baseline \texttt{Llama~3.2~3B Instruct} run. To assess whether these conclusions persist beyond a single checkpoint or a single implementation setting, the appendix reports additional robustness experiments covering cross-model transfer, layerwise variation, downstream seed reproducibility, a random-bases null control, atlas hyperparameter sweeps, graph-construction changes, and transport-regularisation ablations.
\subsection{A: spanning-tree gauge identity and persistence tradeoff}
\label{sec:experiments-A}
Result A makes two claims that are checkable numerically:
(i) after spanning-tree gauge fixing, tree-edge residuals are (numerically) zero,
and (ii) for each chord, the chord residual equals the holonomy of the associated fundamental cycle.

\paragraph{Quantitative identity check (tree residuals and chord/holonomy agreement).}
For each oriented edge we formed $g_{vu}=\widehat P_{vu}^\top Q_{vu}\in O(k)$.
We built a spanning tree on the LCC and constructed the spanning-tree gauge $(U_c)$ from Theorem~\ref{thm:A}.
We then measured:
\[
r_{\text{tree}}:=\|U_v g_{vu}U_u^\top-I\|_F\quad\text{on tree edges},\qquad
r_{\text{chord}}:=\|U_v g_{vu}U_u^\top-I\|_F\quad\text{on chords}.
\]
By Theorem~\ref{thm:A}, chord residuals equal fundamental-cycle holonomy defects $\|h_{\gamma_e}-I\|_F$.
In our runs the mean absolute difference between chord residual and holonomy defect was $\sim 10^{-6}$, i.e.\ numerical precision; tree residuals were $\sim 10^{-5}$ due to floating-point error and orthogonal projection.

Table~\ref{tab:A-identity} reports the identity-check statistics for baseline ($s_{\min}=0$) and persistent ($s_{\min}=0.015$) subsystems.
The key interpretation is:
\emph{the tree residuals are near zero, while chord residuals are large and match holonomy}.
This directly supports the claim that holonomy is not a coordinate artefact: it is exactly what remains after maximal gauge fixing.

\begin{table}[H]
\centering
\caption{Result A identity checks on the LCC ($k=32$). Tree residuals quantify successful gauge fixing on the spanning tree. Chord residuals coincide with fundamental-cycle holonomy as predicted by Theorem~\ref{thm:A}.}
\label{tab:A-identity}
\resizebox{\columnwidth}{!}{
\begin{tabular}{lrrrr}
\toprule
Subsystem & Tree residual mean & Tree residual max & Chord/hol mean $\|h-I\|_F$ & Chord/hol max $\|h-I\|_F$\\
\midrule
Baseline ($s_{\min}=0$) & $1.6003\times 10^{-5}$ & $2.3228\times 10^{-5}$ & $4.4330$ & $6.7364$\\
Persistent ($s_{\min}=0.015$) & $1.9701\times 10^{-5}$ & $3.1521\times 10^{-5}$ & $4.0650$ & $6.3957$\\
\bottomrule
\end{tabular}}
\end{table}

\paragraph{Persistence sweep (conditioning vs connectivity vs holonomy).}
The persistence threshold $s_{\min}$ filters edges by $\sigma_{\min}(\widehat T_{vu})\ge s_{\min}$.
This improves conditioning and, empirically, reduces holonomy magnitudes; but it also removes edges and can collapse cycles.
Table~\ref{tab:persist} quantifies this tradeoff.

The column $E_{\mathrm{LCC}}$ is the number of edges remaining in the LCC after filtering; it measures how much transition information remains connected.
The column \#chords is the number of non-tree edges; it is the number of independent fundamental cycles available for holonomy computation on that LCC.
The last two columns report mean and max normalised holonomy defect $D_{\mathrm{hol}}=\|h-I\|_F/\sqrt{2k}$.
The numerical pattern shows:
\begin{itemize}[leftmargin=2em]
\item As $s_{\min}$ increases, mean $D_{\mathrm{hol}}$ decreases from $0.554$ at $s_{\min}=0$ to $0.456$ at $s_{\min}=0.02$, suggesting that poorly conditioned edges inflate apparent curvature.
\item The number of available cycles decreases sharply (84 chords at baseline to 17 at $s_{\min}=0.02$), meaning aggressive filtering can make holonomy estimates less representative simply because fewer loops remain.
\end{itemize}

\begin{table}[H]
\centering
\caption{Persistence sweep ($k=32$). Increasing $s_{\min}$ improves conditioning but reduces connectivity and cycle availability. The reported $D_{\mathrm{hol}}$ values are computed on fundamental cycles induced by chords in the LCC.}
\label{tab:persist}
\begin{tabular}{lrrrrr}
\toprule
$s_{\min}$ & $E_{\mathrm{LCC}}$ & LCC size & \#chords & mean $D_{\mathrm{hol}}$ & max $D_{\mathrm{hol}}$\\
\midrule
0.0000 & 173 & 90 & 84 & 0.554 & 0.842\\
0.0100 & 129 & 84 & 46 & 0.517 & 0.799\\
0.0125 & 106 & 73 & 34 & 0.517 & 0.799\\
0.0150 &  90 & 62 & 29 & 0.508 & 0.799\\
0.0200 &  75 & 59 & 17 & 0.456 & 0.789\\
\bottomrule
\end{tabular}
\end{table}

\paragraph{Distributional view of holonomy (figure interpretation).}
Figure~\ref{fig:A-hol} complements Table~\ref{tab:persist} by showing the full distribution of $D_{\mathrm{hol}}$ values for baseline and persistent subsystems.
The histogram shows how mass shifts; the ECDF highlights tail behaviour (high-holonomy loops).
In this run the persistent distribution is shifted toward smaller values, but retains a substantial upper tail, indicating that improved conditioning does not eliminate genuine path dependence.

\begin{figure}[H]
\centering
\includegraphics[width=0.48\linewidth]{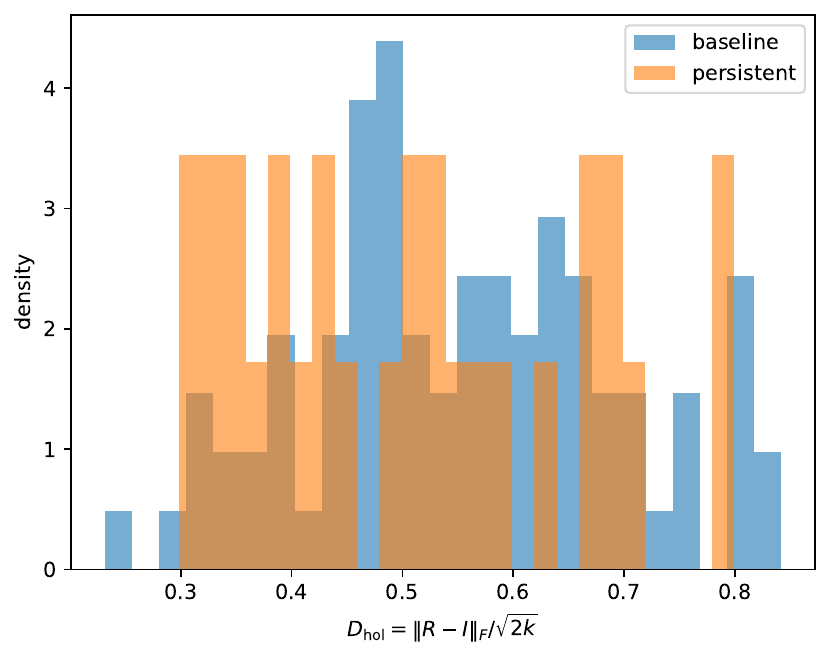}
\includegraphics[width=0.48\linewidth]{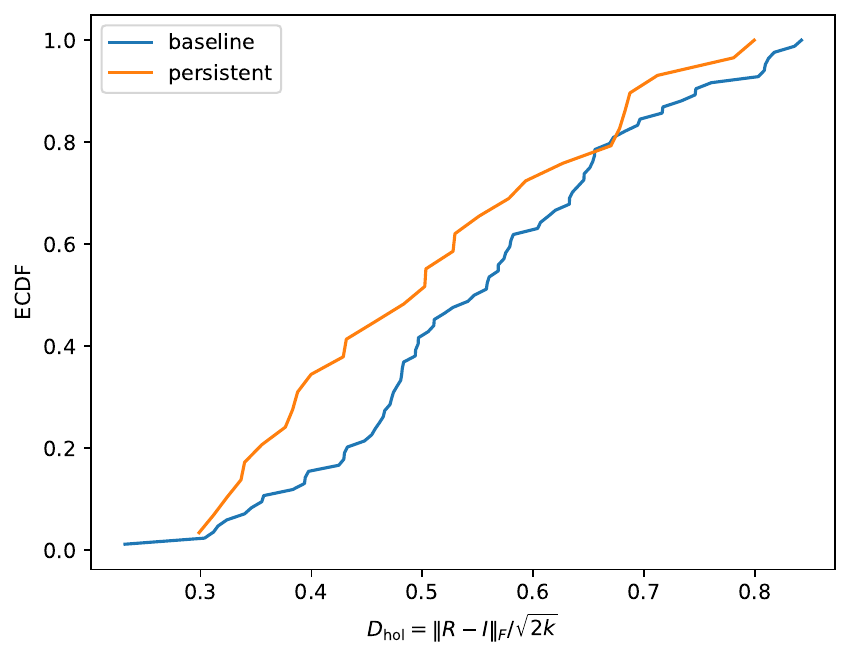}
\caption{Holonomy defect distributions (normalised) comparing baseline ($s_{\min}=0$) and persistent ($s_{\min}=0.015$) subsystems.}
\label{fig:A-hol}
\end{figure}

\subsection{B: empirical validation of Theorem B (shearing bound)}
\label{sec:experiments-B}

Theorem~\ref{thm:B} is a deterministic inequality. Empirically, we evaluate its instantiated quantities on each edge:
\[
\widehat\Delta_{uv}=\frac{1}{n}\|(Q_{vu}-\widehat P_{vu})Z_u\|_F^2,\qquad
\widehat{\mathrm{LB}}_{uv}=\lambda_{\min}(\widehat\Sigma_u)\|Q_{vu}-\widehat P_{vu}\|_F^2,
\]
where $\widehat\Sigma_u=\frac{1}{n}Z_u Z_u^\top$.
We report slack $\widehat\Delta_{uv}/\widehat{\mathrm{LB}}_{uv}$ (with numerical stabilisation).
Slack should satisfy $\widehat\Delta_{uv}/\widehat{\mathrm{LB}}_{uv}\ge 1$ up to numerical error when $\widehat{\mathrm{LB}}_{uv}>0$; values much larger than $1$ indicate that the bound is conservative.

\paragraph{Edge-level distributions (numbers and meaning).}
Table~\ref{tab:B-stats} summarises two aspects of Result B:
(i) the distribution of $D_{\mathrm{shear}}$ (how large the geometric/semantic mismatch is), and
(ii) the distribution of slack (how conservative the theoretical bound is on real overlap covariance).
For the baseline subsystem, mean $D_{\mathrm{shear}}$ is $0.2015$ with max $0.3743$.
For the persistent subsystem, mean $D_{\mathrm{shear}}$ is $0.1910$ with max $0.3730$.
This indicates that conditioning-based filtering slightly reduces average mismatch, but does not remove large-shearing edges.
Slack medians are around $7.8$, meaning $\widehat\Delta_{uv}$ is typically an order of magnitude larger than the eigenvalue-based lower bound; this is expected when $\lambda_{\min}(\widehat\Sigma_u)$ is small.

\begin{table}[H]
\centering
\caption{Result B summary statistics for shearing and slack (normalised) computed over all retained edges in the subsystem. Slack is $\widehat\Delta_{uv}/\widehat{\mathrm{LB}}_{uv}$; values $>1$ confirm the bound (up to numerical error), and large values quantify conservativeness due to small $\lambda_{\min}(\widehat\Sigma_u)$.}
\label{tab:B-stats}
\begin{tabular}{lrrrrrrrr}
\toprule
Subsystem & $n_{\text{edges}}$ & $D_{\min}$ & $D_{\text{med}}$ & $D_{\text{mean}}$ & $D_{\max}$ & slack$_{\min}$ & slack$_{\text{med}}$ & slack$_{\max}$\\
\midrule
Baseline & 173 & 0.0354 & 0.2080 & 0.2015 & 0.3743 & 2.314 & 7.802 & 399.525\\
Persistent & 90 & 0.0637 & 0.1966 & 0.1910 & 0.3730 & 2.565 & 7.837 & 399.525\\
\bottomrule
\end{tabular}
\end{table}

\paragraph{Figure interpretation (why each panel exists).}
Figures~\ref{fig:B-base} and \ref{fig:B-pers} each contain three panels:
\begin{enumerate}[leftmargin=2em]
\item \emph{$\widehat\Delta$ vs $\widehat{\mathrm{LB}}$:} points above the diagonal confirm the inequality edge-by-edge.
\item \emph{$\widehat\Delta$ vs $D_{\mathrm{shear}}$:} demonstrates that larger mismatch scores correspond to larger realised transfer mismatch energies.
\item \emph{Slack histogram:} quantifies how conservative the bound is; heavy tails arise when some edges have extremely small $\lambda_{\min}(\widehat\Sigma_u)$.
\end{enumerate}
In combination, these panels demonstrate (i) correctness, (ii) monotone usefulness, and (iii) the bound’s tightness regime.
For visualisation we exclude edges with near-zero $\lambda_{\min}(\widehat\Sigma_u)$ (e.g.\ $\lambda_{\min}<10^{-6}$), since the bound becomes numerically trivial and slack is dominated by stabilisation constants.

\begin{figure}[H]
\centering
\includegraphics[width=0.32\linewidth]{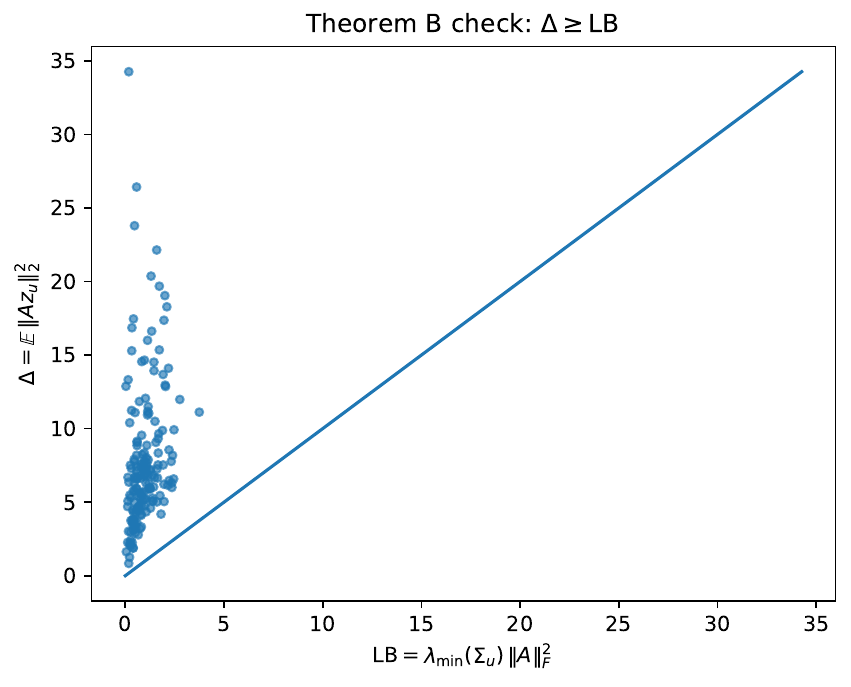}
\includegraphics[width=0.32\linewidth]{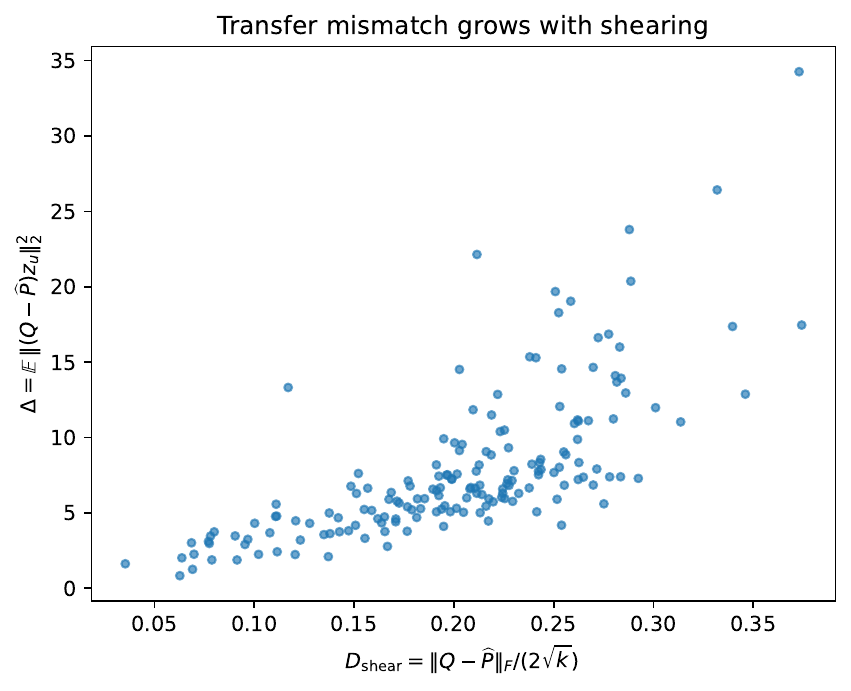}
\includegraphics[width=0.32\linewidth]{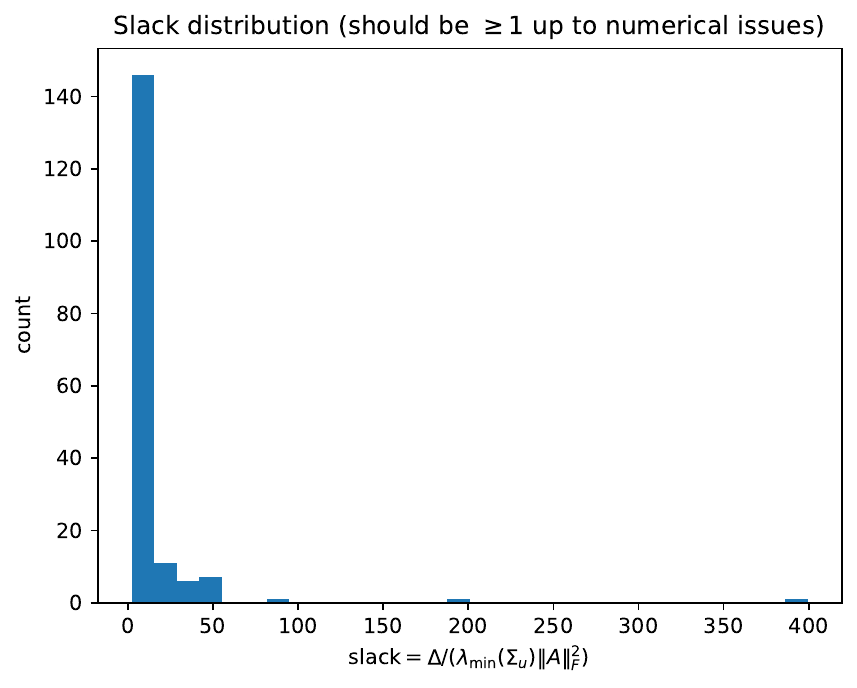}
\caption{Baseline subsystem: (left) $\widehat\Delta_{uv}$ vs lower bound $\widehat{\mathrm{LB}}_{uv}$ (points above diagonal), (middle) $\widehat\Delta_{uv}$ vs $D_{\mathrm{shear}}(u,v)$, (right) slack distribution.}
\label{fig:B-base}
\end{figure}

\begin{figure}[H]
\centering
\includegraphics[width=0.32\linewidth]{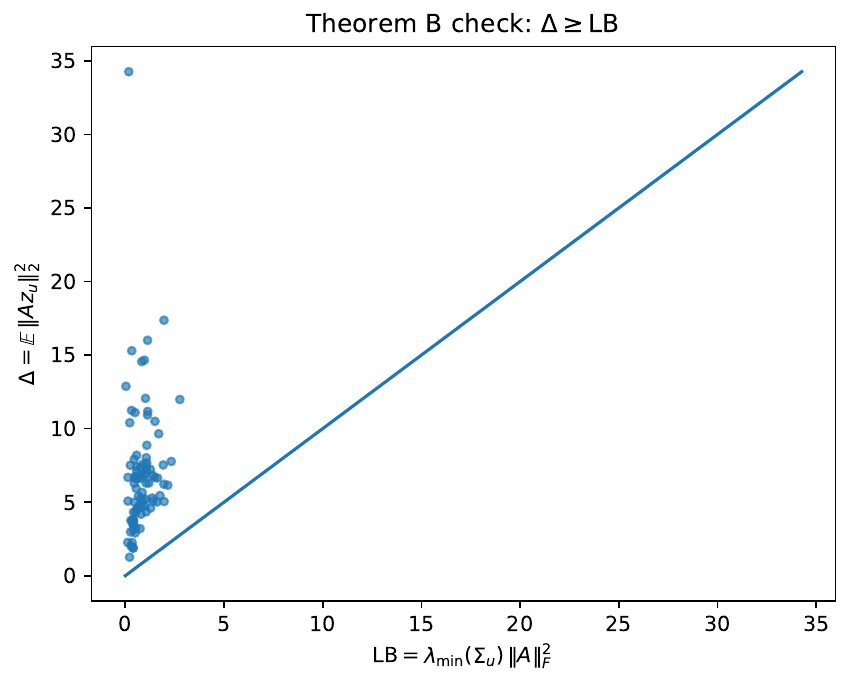}
\includegraphics[width=0.32\linewidth]{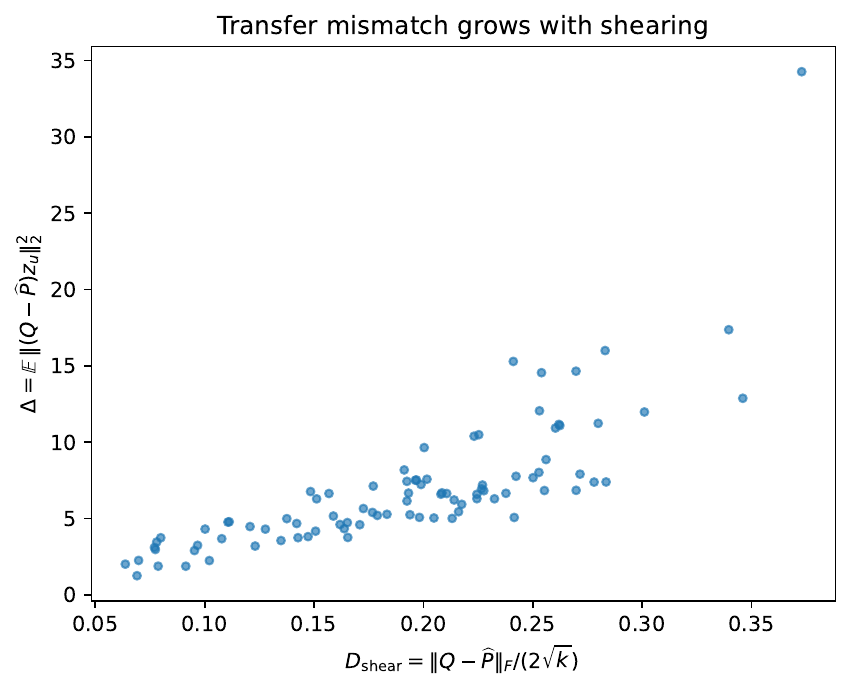}
\includegraphics[width=0.32\linewidth]{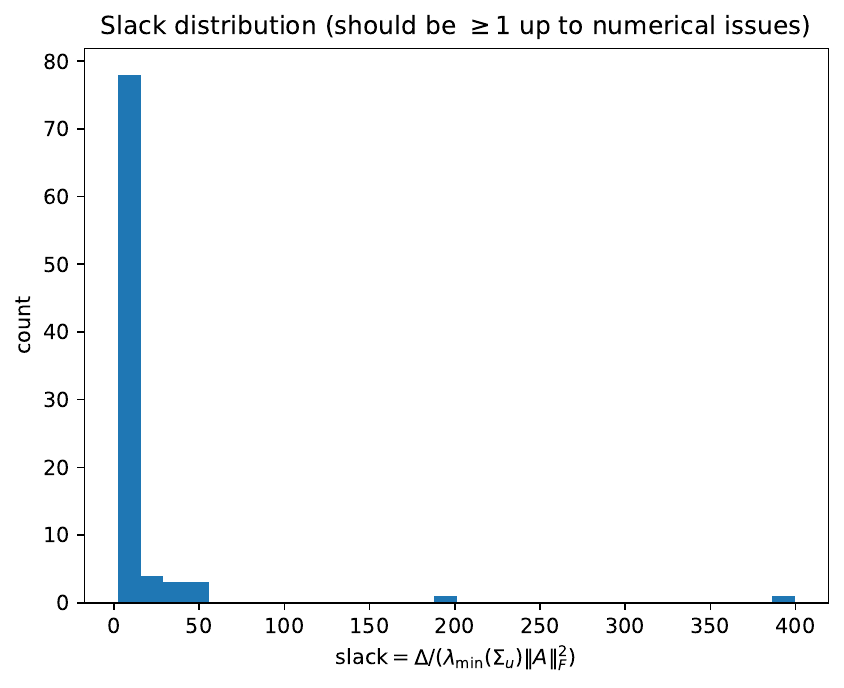}
\caption{Persistent subsystem ($s_{\min}=0.015$): the same diagnostics as Figure~\ref{fig:B-base}.}
\label{fig:B-pers}
\end{figure}

\subsection{C: certified jamming/interference (coverage, correctness, robustness)}
\label{sec:experiments-C}

Theorem~\ref{thm:C} is a deterministic certificate once the consequential subset $A$ and threshold $\tau_\star$ are fixed.
Empirically, the key scientific questions are:
\begin{enumerate}[leftmargin=2em]
\item \textbf{Non-vacuity (coverage):} how often does the procedure find $A$ with $\widehat{\mathrm{LB}}(c)>0$?
\item \textbf{Correctness:} are there any violations of $\mc E_A^{(r)}(c)\ge \widehat{\mathrm{LB}}(c)$ (slack $<1$)?
\item \textbf{Robustness:} are coverage/correctness stable across random seeds and dictionary hyperparameters?
\end{enumerate}

\paragraph{What the reported quantities mean.}
Cert.\ charts is the number of charts (out of 40 analysed) with $\widehat{\mathrm{LB}}(c)>0$; this is the non-vacuity measure.
Slack is $\mc E_A^{(r)}(c)/\widehat{\mathrm{LB}}(c)$ on certified charts; slack must be $\ge 1$.
Slack minimum therefore acts as a stringent correctness check.
Corr$(J,\mc E^{(r)})$ and Corr$(J,\mc E_A^{(r)})$ measure whether the diagnostic $J(c)$ tracks observed interference energies; these correlations are empirical and not part of the certificate.

\paragraph{Seed stability.}
Table~\ref{tab:C-seeds} shows that for $m=256,\alpha=1.0$ the certification rate is $0.85$--$0.875$ across seeds, and slack minima remain comfortably above $1$ (no violations).
These numbers mean that the certificate is not a rare event: most charts admit a nontrivial lower bound, and the algorithmic construction is not sensitive to random initialisation.

\begin{table}[H]
\centering
\caption{Seed stability for Result C (fixed $m=256$, $\alpha=1.0$; 40 charts). Cert.\ rate measures non-vacuity. Slack min $>1$ indicates zero violations. Correlations quantify the empirical association between jamming and interference.}
\label{tab:C-seeds}
\begin{tabular}{lrrrrrr}
\toprule
Seed & Cert.\ charts & Cert.\ rate & Slack median & Slack min & Corr$(J,\mc E^{(r)})$ & Corr$(J,\mc E_A^{(r)})$\\
\midrule
0 & 35 & 0.875 & 7.158 & 1.204 & 0.634 & 0.539\\
1 & 34 & 0.850 & 7.580 & 1.205 & 0.601 & 0.575\\
2 & 34 & 0.850 & 7.400 & 1.185 & 0.624 & 0.630\\
3 & 35 & 0.875 & 7.277 & 1.177 & 0.624 & 0.592\\
4 & 35 & 0.875 & 8.157 & 1.181 & 0.607 & 0.534\\
\bottomrule
\end{tabular}
\end{table}

\paragraph{Hyperparameter robustness.}
Table~\ref{tab:C-hparams} varies dictionary size $m$ and sparsity parameter $\alpha$.
Certification remains high ($0.85$--$0.90$), and slack minima remain $>1$ in all settings.
This indicates that the certificate is not tied to one specific dictionary configuration, which is essential if jamming is to be treated as a structural property of the chart rather than a modelling artefact.

\begin{table}[H]
\centering
\caption{Hyperparameter robustness for Result C (40 charts). Coverage and correctness are stable across dictionary sizes and sparsity strengths.}
\label{tab:C-hparams}
\begin{tabular}{lrrrrrr}
\toprule
$(m,\alpha)$ & Cert.\ charts & Cert.\ rate & Slack median & Slack min & Corr$(J,\mc E^{(r)})$ & Corr$(J,\mc E_A^{(r)})$\\
\midrule
(128, 0.5) & 35 & 0.875 & 6.785 & 1.198 & 0.464 & 0.464\\
(128, 1.0) & 36 & 0.900 & 6.801 & 1.198 & 0.538 & 0.550\\
(256, 0.5) & 34 & 0.850 & 6.746 & 1.204 & 0.497 & 0.477\\
(256, 1.0) & 35 & 0.875 & 7.158 & 1.204 & 0.634 & 0.539\\
(512, 0.5) & 34 & 0.850 & 7.054 & 1.202 & 0.424 & 0.292\\
(512, 1.0) & 34 & 0.850 & 8.566 & 1.202 & 0.662 & 0.620\\
\bottomrule
\end{tabular}
\end{table}

\begin{figure}[H]
\centering
\includegraphics[width=0.32\linewidth]{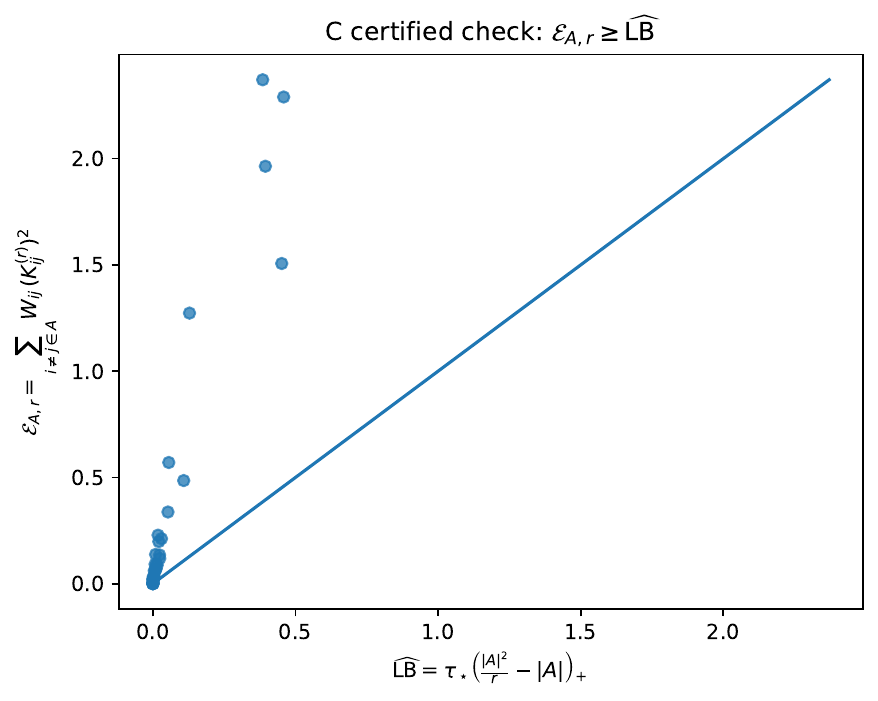}
\includegraphics[width=0.32\linewidth]{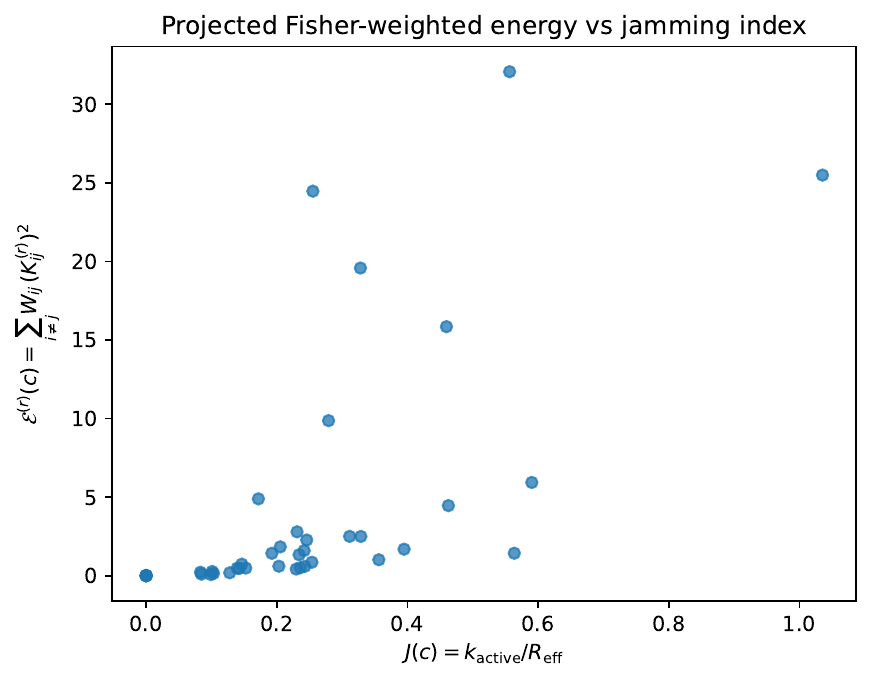}
\includegraphics[width=0.32\linewidth]{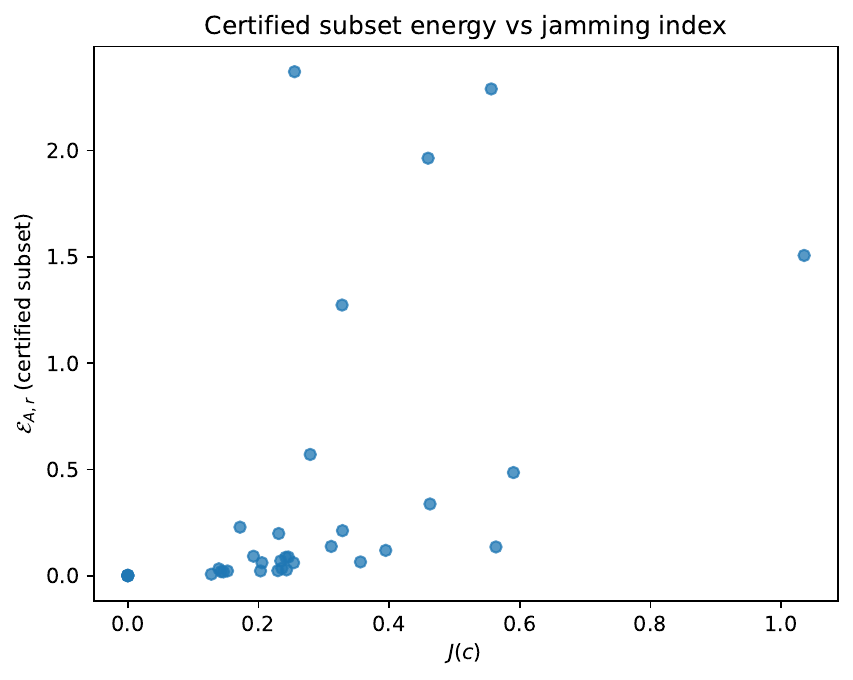}
\caption{Result C diagnostics (PDF). Left: certificate check $\mc E_A^{(r)}\ge \widehat{\mathrm{LB}}$. Middle: projected interference energy vs jamming index $J(c)$. Right: certified subset energy vs $J(c)$.}
\label{fig:C}
\end{figure}

\paragraph{Figure interpretation (certificate and correlation diagnostics).}
Figure~\ref{fig:C} contains three plots that answer three distinct questions:
\begin{enumerate}[leftmargin=2em]
\item \emph{Certificate check:} $\mc E_A^{(r)}$ vs $\widehat{\mathrm{LB}}$. Points must lie above the diagonal; this confirms the inequality numerically for all certified charts.
\item \emph{Projected energy vs jamming:} tests whether $J(c)$ tracks the measured interference energy on the bandwidth-projected subspace.
\item \emph{Certified energy vs jamming:} tests the same relationship restricted to the consequential subset used for certification.
\end{enumerate}
Together they show both that the certificate holds and that the jamming index has predictive value as an interference proxy.

\subsection{D: stability of $D_{\mathrm{shear}}$ and $D_{\mathrm{hol}}$ under resampling}
\label{sec:experiments-D}

Result D is about estimator behaviour under finite overlap sampling.
We use bootstrap resampling to quantify variability when overlap sets are resampled with replacement \cite{efron1979bootstrap}.
We report (i) global bootstrap summaries across many edges/loops and (ii) within-edge and within-loop sample-size curves that isolate concentration with increasing $n_{\mathrm{boot}}$.

\paragraph{Global bootstrap summary (distributional stability).}
We bootstrap $D_{\mathrm{shear}}$ and $D_{\mathrm{hol}}$ by resampling overlap sets with replacement.

For $D_{\mathrm{shear}}$, each bootstrap replicate resamples $n_{\mathrm{boot}}$ overlap points per selected edge and recomputes
$Q_{vu}=\operatorname{polar}(\widehat T_{vu})$ and $\widehat P_{vu}=\operatorname{polar}(B_v^\top B_u)$, yielding
$D_{\mathrm{shear}}(u,v)=\|Q_{vu}-\widehat P_{vu}\|_F/(2\sqrt{k})$.

For $D_{\mathrm{hol}}$, we sample fundamental cycles induced by chords in a spanning tree of the LCC: for each chord $(u,v)$ we close the unique
tree path from $u$ to $v$ with the chord step $v\to u$.
In each replicate we estimate the required per-edge $(Q_{vu},\widehat P_{vu})$ from resampled overlaps, form $g_{vu}=\widehat P_{vu}^\top Q_{vu}$, then compute the
loop product $h_\gamma=\prod g$ and $D_{\mathrm{hol}}(\gamma)=\|h_\gamma-I\|_F/\sqrt{2k}$.
Because overlap sizes vary by edge, some replicate--cycle evaluations may be undefined if one or more cycle edges lack sufficient overlap; the reported
global holonomy statistics therefore aggregate only those replicate--cycle evaluations for which all required cycle edges were successfully estimated
(hence the realised global \#samples can be $<B\times n_{\text{cycles}}$).

\begin{table}[H]
\centering
\caption{Global bootstrap stability summaries (normalised). Each row aggregates many bootstrap replicates across sampled edges (shearing) or cycles (holonomy). Quantiles describe the distributional tails of the measured defects.}
\label{tab:D-global}
\begin{tabular}{lrrrrrrr}
\toprule
Subsystem & Metric & \#samples & mean & std & q05 & q50 & q95\\
\midrule
Baseline & $D_{\mathrm{shear}}$ & 10000 & 0.2403 & 0.0689 & 0.0966 & 0.2463 & 0.3334\\
Baseline & $D_{\mathrm{hol}}$   & 1200  & 0.5772 & 0.0977 & 0.4277 & 0.5777 & 0.7386\\
\addlinespace
Persistent ($s_{\min}=0.015$) & $D_{\mathrm{shear}}$ & 10000 & 0.2033 & 0.0629 & 0.0845 & 0.2104 & 0.2951\\
Persistent ($s_{\min}=0.015$) & $D_{\mathrm{hol}}$   & 4200  & 0.5454 & 0.1355 & 0.3875 & 0.5038 & 0.8018\\
\bottomrule
\end{tabular}
\end{table}

\paragraph{Within-edge concentration (estimator variance at fixed edge).}
Tables~\ref{tab:D-edge-base} and \ref{tab:D-edge-pers} fix a single edge and vary $n_{\mathrm{boot}}$.
This isolates estimator concentration: the standard deviation in Table~\ref{tab:D-edge-pers} decreases from $0.00606$ at $n=256$ to $0.00169$ at $n=2736$, which is consistent with Lemma~\ref{lem:polar} (better conditioning implies smaller variability of polar factors and therefore of $D_{\mathrm{shear}}$).

\begin{table}[H]
\centering
\caption{Within-edge stability for $D_{\mathrm{shear}}$ on baseline edge (3--33), $B=200$.}
\label{tab:D-edge-base}
\begin{tabular}{rrr}
\toprule
$n_{\mathrm{boot}}$ & mean $D_{\mathrm{shear}}$ & std\\
\midrule
256  & 0.69348 & 0.00760\\
512  & 0.70326 & 0.00515\\
1024 & 0.70058 & 0.00588\\
2000 & 0.70735 & 0.00537\\
4000 & 0.70086 & 0.00824\\
\bottomrule
\end{tabular}
\end{table}

\begin{table}[H]
\centering
\caption{Within-edge stability for $D_{\mathrm{shear}}$ on persistent edge (68--83), $B=200$ (last row capped by available overlap). The decreasing std indicates concentration with increasing overlap size.}
\label{tab:D-edge-pers}
\begin{tabular}{rrr}
\toprule
$n_{\mathrm{boot}}$ & mean $D_{\mathrm{shear}}$ & std\\
\midrule
256  & 0.07365 & 0.00606\\
512  & 0.05668 & 0.00394\\
1024 & 0.04654 & 0.00247\\
2000 & 0.04146 & 0.00193\\
2736 & 0.03994 & 0.00169\\
\bottomrule
\end{tabular}
\end{table}

\paragraph{Within-loop concentration (cycle-level stability).}
Tables~\ref{tab:D-loop-base} and \ref{tab:D-loop-pers} fix a single fundamental cycle and vary $n_{\mathrm{boot}}$.
The standard deviation remains on the order of $2\times10^{-2}$ to $3\times10^{-2}$, which is expected because holonomy aggregates multiple edges (Lemma~\ref{lem:loop}).
The mean values stabilise after $n=512$ in both examples, supporting that $D_{\mathrm{hol}}$ is not a fragile artefact of a particular overlap subsample.

\begin{table}[H]
\centering
\caption{Within-loop stability for $D_{\mathrm{hol}}$ on a baseline fundamental cycle (chord 40--102), $B=200$.}
\label{tab:D-loop-base}
\begin{tabular}{rrr}
\toprule
$n_{\mathrm{boot}}$ & mean $D_{\mathrm{hol}}$ & std\\
\midrule
256  & 0.48596 & 0.03030\\
512  & 0.44759 & 0.02893\\
1024 & 0.44692 & 0.02899\\
2000 & 0.44595 & 0.02749\\
4000 & 0.44489 & 0.02925\\
\bottomrule
\end{tabular}
\end{table}

\begin{table}[H]
\centering
\caption{Within-loop stability for $D_{\mathrm{hol}}$ on a persistent fundamental cycle (chord 71--96), $B=200$.}
\label{tab:D-loop-pers}
\begin{tabular}{rrr}
\toprule
$n_{\mathrm{boot}}$ & mean $D_{\mathrm{hol}}$ & std\\
\midrule
256  & 0.56800 & 0.02492\\
512  & 0.54174 & 0.02157\\
1024 & 0.52724 & 0.02404\\
2000 & 0.52747 & 0.02169\\
4000 & 0.52957 & 0.02452\\
\bottomrule
\end{tabular}
\end{table}

\paragraph{Conclusion for D.}
Across both edge-level and loop-level tests, the reported diagnostics stabilise under resampling at practical overlap sizes.
Persistent filtering improves within-edge concentration, consistent with polar-factor perturbation theory.

\FloatBarrier

\section{Obstructions and atlas-level summaries}

\begin{definition}[Three geometric obstructions]
\begin{enumerate}[label=\textbf{O\arabic*:}, leftmargin=2.2em]
\item \textbf{Local jamming:} large $J(c)=k_{\mathrm{active}}(c)/R_{\mathrm{eff}}(G^{(c)})$ indicates load exceeds bandwidth.
\item \textbf{Proxy shearing:} large $D_{\mathrm{shear}}(u,v)$ implies unavoidable mismatch by Theorem~\ref{thm:B}.
\item \textbf{Nontrivial holonomy:} large $D_{\mathrm{hol}}(\gamma)$ obstructs global trivialisation by Theorem~\ref{thm:hol-obs}.
\end{enumerate}
\end{definition}

Atlas-level aggregates can be reported as means and quantiles:
\[
\mc J_{\mathrm{atlas}}=\E_c[J(c)],\quad
\mc S_{\mathrm{atlas}}=\E_{(u,v)}[D_{\mathrm{shear}}(u,v)],\quad
\mc H_{\mathrm{atlas}}=\E_{\gamma}[D_{\mathrm{hol}}(\gamma)].
\]
In this paper, the emphasis is on establishing that the edge and loop diagnostics are well-defined, certified when claimed, and stable under resampling.

\FloatBarrier

\section{Limitations}

\begin{enumerate}[leftmargin=2em]
\item \textbf{Correspondence-proxy dependence.} $\widehat P_{vu}$ is a proxy; conclusions about ``mismatch'' are conditional on the proxy definition.
\item \textbf{Transport model class.} We use linear transports on overlap codes as a deliberately simple and reproducible baseline that keeps the transport estimation, polar-factor analysis, and holonomy computation analytically transparent. Nonlinear transports may be required for certain strata, and extending the framework to richer local map classes is an important direction for future work.
\item \textbf{No topological overclaim.} Nontrivial holonomy obstructs a connection-compatible trivialisation under the learned transports (with respect to the fixed correspondence proxy); it does not prove topological nontriviality of an underlying continuous bundle.
\item \textbf{Scaling constraints.} Full Fisher matrices can be expensive; scalable variants (diagonal/block/low-rank) should be used for large $m_c$.
\item \textbf{Robustness scope.} The appendix evaluates cross-model transfer, layerwise variation, downstream seed stability, a random-bases null control, and moderate atlas/graph/transport ablations. However, we do not exhaust all possible proxy choices, nonlinear transport classes, or downstream task-level operational validations, which remain directions for future work.
\end{enumerate}

\section{Conclusion}
We presented a discrete gauge-theoretic atlas formalism for superposition in LLMs and grounded it in four results with concrete empirical instantiations.
Result A establishes a constructive gauge identity connecting chord residuals to holonomy, making holonomy computable and gauge-invariant.
Result B turns proxy shearing into an unavoidable transfer mismatch bound.
Result C provides a non-vacuous certified interference bound on consequential atom subsets with high coverage and zero violations across seeds/hyperparameters.
Result D shows that the central diagnostics are stable under resampling and concentrate with overlap size, particularly on persistent subsystems.
Together, these results support interpretability as an atlas construction problem with measurable local and global obstructions. The appendix further shows that these diagnostics are not restricted to a single baseline run: they transfer across multiple model families, vary systematically across depth, remain stable across downstream seeds, respond strongly to a random-bases null control, and are robust to moderate changes in atlas, graph, and transport construction. This strengthens the case that the framework is not merely a formal reinterpretation, but an empirically executable methodology for diagnosing when and why global interpretability fails.

\subsubsection*{Broader Impact Statement}
This paper introduces a new mathematical and empirical framework for mechanistic interpretability: a gauge- and sheaf-theoretic \emph{atlas} view of neural representations that replaces the common “single global dictionary” premise with a structured local-to-global objective. 
The core impact is conceptual and methodological: we define measurable geometric obstructions to global interpretability (local jamming, proxy shearing, and nontrivial holonomy), provide constructive gauge computations that make holonomy directly computable, and supply certified and stability-tested diagnostics that can be applied to frozen models. 
If adopted, this framework can change how interpretability research formulates its goals (from global feature discovery to atlas construction), and it can enable more rigorous evaluation of when and why interpretability methods fail due to context dependence, interference, and path-dependent transport.

As with most interpretability work, these tools are dual-use. They can be used to audit, debug, and improve model transparency, but they may also help an adversary identify representation regimes with large mismatch or curvature and exploit them to induce or amplify undesirable behaviours. 
The paper does not introduce new training procedures or increase model capability; it provides analysis tools and measurements on an existing frozen model. 
Responsible use includes applying these diagnostics to safety evaluation and robustness testing, and exercising caution when publishing fine-grained analyses that might facilitate targeted attacks.

\subsubsection*{Author Contributions}
Hossein Javidnia conceived the study; developed the theoretical framework; designed and implemented the experimental pipeline; conducted the experiments; analysed and interpreted the results; and wrote the manuscript.

\subsubsection*{Acknowledgments}
The author thanks the ADAPT Centre for providing access to high-performance computing (HPC) resources used for the experimental runs. The ADAPT Centre did not contribute to the study design, analysis, interpretation, or writing.

\bibliography{main}
\bibliographystyle{tmlr}

\newpage
\appendix

\paragraph{Appendix overview.}
This appendix collects seven additional empirical sections that broaden the scope of the main-text experiments. Their purpose is to test whether the proposed diagnostics are specific to one baseline run or instead reflect broader structural properties of the learned representation. The appendix is organised as follows. Section~\ref{app:model_generalisation} reports cross-model generalisation experiments on \texttt{Qwen~2.5~3B Instruct} and \texttt{Gemma~2~2B~IT}. Section~\ref{app:layerwise_generalisation} studies layerwise variation within the baseline Llama model. Section~\ref{app:seed_reproducibility} reports downstream seed-reproducibility results, and Section~\ref{app:null_random_bases} gives a random-bases null control. Sections~\ref{app:hparam_robustness}, \ref{app:graph_construction_robustness}, and \ref{app:transport_ablation} then examine robustness to atlas hyperparameters, graph construction, and transport regularisation, respectively. Taken together, these sections extend the empirical evidence for the framework while leaving the theoretical claims and main-text core results unchanged.

\section{Additional model-generalisation experiments}
\label{app:model_generalisation}

To test whether the proposed atlas-based pipeline is specific to a single checkpoint, we repeated the core analysis on two additional instruction-tuned models while keeping the remainder of the configuration aligned with the baseline setup used in the main experiments. In addition to the baseline model \texttt{meta-llama/Llama-3.2-3B-Instruct}, we ran the full pipeline on \texttt{Qwen/Qwen2.5-3B-Instruct} and \texttt{google/gemma-2-2b-it}. In all cases we used the same three data strata (WikiText-103, C4-derived web text, and \texttt{the-stack-smol}), the same hidden layer ($\ell=16$), the same activation sampling budget ($200{,}000$ tokens), the same atlas size ($C=128$ clusters), the same local support dimension ($k=32$), and the same downstream settings for transport, holonomy, shearing, and certified jamming. These experiments therefore isolate the effect of changing the model while leaving the rest of the atlas construction and evaluation pipeline fixed.

Both additional runs completed successfully and produced nondegenerate outputs for all three main diagnostics. First, the gauge sanity check continued to hold across models: tree-edge residuals remained numerically near zero, while chord residuals matched holonomy defects to numerical precision. Second, the shearing analysis produced a substantial set of usable edges in all runs together with clear nonzero shearing scores. Third, the certified jamming analysis remained nontrivial, with positive certification on most analysed clusters and no certification violations in any run. Overall, these experiments show that the framework transfers beyond the baseline Llama checkpoint, while also indicating that the resulting geometric and interference statistics are model-dependent rather than architecture-invariant.

\subsection{Experimental configuration}
\label{app:model_generalisation_setup}

Table~\ref{tab:model_generalisation_setup} summarises the fixed configuration used for the model-comparison experiments. The only quantity that changes across the three runs is the model checkpoint; all data, atlas, transport, and jamming settings are otherwise matched to the baseline pipeline.

\begin{table}[h]
\centering
\small
\begin{tabular}{ll}
\hline
\textbf{Setting} & \textbf{Value} \\
\hline
Baseline model & \texttt{meta-llama/Llama-3.2-3B-Instruct} \\
Additional model 1 & \texttt{Qwen/Qwen2.5-3B-Instruct} \\
Additional model 2 & \texttt{google/gemma-2-2b-it} \\
Layer & $16$ \\
Activation budget & $200{,}000$ tokens \\
Sequence length & $256$ \\
Sampling stride & every $8$ tokens \\
Number of clusters & $128$ \\
Local basis dimension & $k=32$ \\
kNN graph degree & $6$ \\
Ridge parameter & $\lambda = 10^{-2}$ \\
Minimum overlap & $256$ \\
Maximum overlap & $8000$ \\
Gradient samples per cluster & $512$ \\
Jamming analysis clusters & $40$ \\
Jamming dictionary width & $m=256$ \\
Jamming sparsity parameter & $\alpha = 1.0$ \\
\hline
\end{tabular}
\caption{Configuration used for the model-generalisation experiments. All settings were matched to the baseline pipeline except for the model checkpoint.}
\label{tab:model_generalisation_setup}
\end{table}

\subsection{Holonomy comparison}
\label{app:model_generalisation_holonomy}

Table~\ref{tab:model_generalisation_holonomy} summarises the graph-level holonomy results for the baseline Llama run and the two additional model runs. It reports both the underlying graph structure and the resulting holonomy statistics, allowing the reader to distinguish changes in cycle geometry from changes in numerical gauge stability.

\begin{table*}[t]
\centering
\small
\begin{tabular}{lccc}
\hline
\textbf{Metric} & \textbf{Llama 3.2 3B Instruct} & \textbf{Qwen 2.5 3B Instruct} & \textbf{Gemma 2 2B IT} \\
\hline
Connected components & 7 & 10 & 6 \\
Largest connected component size & 90 & 78 & 106 \\
Usable edges in LCC & 173 & 138 & 185 \\
Tree edges & 89 & 77 & 105 \\
Chord edges & 84 & 61 & 80 \\
Tree residual mean & $1.60\times 10^{-5}$ & $2.05\times 10^{-5}$ & $2.00\times 10^{-5}$ \\
Tree residual max & $2.32\times 10^{-5}$ & $2.66\times 10^{-5}$ & $3.04\times 10^{-5}$ \\
Chord residual mean & 4.433 & 5.522 & 4.370 \\
Chord residual max & 6.736 & 7.150 & 6.324 \\
Holonomy defect mean & 4.433 & 5.522 & 4.370 \\
Holonomy defect max & 6.736 & 7.150 & 6.324 \\
Normalised holonomy mean $D_{\mathrm{hol}}$ & 0.554 & 0.690 & 0.546 \\
Normalised holonomy max $D_{\mathrm{hol}}$ & 0.842 & 0.894 & 0.790 \\
Residual minus holonomy mean & $2.84\times 10^{-7}$ & $4.80\times 10^{-7}$ & $7.54\times 10^{-7}$ \\
Residual minus holonomy max & $1.77\times 10^{-6}$ & $1.78\times 10^{-6}$ & $3.86\times 10^{-6}$ \\
\hline
\end{tabular}
\caption{Holonomy comparison across the baseline Llama run and the two additional model-generalisation runs. In all cases tree residuals remain numerically negligible, while chord residuals and holonomy defects agree to numerical precision. Qwen exhibits the largest mean holonomy defect and the largest mean normalised holonomy, while Gemma remains close to the Llama baseline.}
\label{tab:model_generalisation_holonomy}
\end{table*}

The holonomy comparison shows that the discrete gauge construction remains well behaved under model substitution. As shown in Table~\ref{tab:model_generalisation_holonomy}, the spanning-tree residual is essentially zero up to numerical precision in all three runs, and the difference between chord residuals and explicit holonomy defects remains negligible. This means that the gauge-fixing construction and the associated cycle computation remain numerically stable across model families. At the same time, the magnitude of the cycle inconsistency varies meaningfully by model: Table~\ref{tab:model_generalisation_holonomy} shows that Qwen exhibits the largest mean holonomy defect and the largest mean normalised holonomy $D_{\mathrm{hol}}$, whereas Gemma remains much closer to the Llama baseline. Thus, the experiment supports two conclusions simultaneously: the construction is portable across models, and the resulting holonomy statistics are sensitive to model-specific representational geometry.

\subsection{Certified jamming comparison}
\label{app:model_generalisation_jamming}

Table~\ref{tab:model_generalisation_jamming} summarises the certified jamming results across the three model families. It reports certification coverage, slack statistics, and the empirical relation between the jamming index and projected interaction energy.

\begin{table*}[t]
\centering
\small
\begin{tabular}{lccc}
\hline
\textbf{Metric} & \textbf{Llama 3.2 3B Instruct} & \textbf{Qwen 2.5 3B Instruct} & \textbf{Gemma 2 2B IT} \\
\hline
Clusters analysed & 40 & 40 & 40 \\
Clusters with positive certified bound & 35 & 38 & 40 \\
Certification rate & 87.5\% & 95.0\% & 100.0\% \\
Certified slack min & 1.204 & 2.061 & 1.980 \\
Certified slack median & 7.158 & 6.834 & 3.368 \\
Certified slack mean & 10.938 & 9.663 & 4.038 \\
Certified slack max & 47.895 & 56.981 & 19.381 \\
Violations $(\mathrm{slack} < 1 - \mathrm{tol})$ & 0 & 0 & 0 \\
$\mathrm{Corr}(J, E_{\mathrm{full\_proj}})$ & 0.634 & 0.026 & -0.080 \\
$\mathrm{Corr}(J, E_{\mathrm{cert}})$ on certified subset & 0.539 & 0.032 & -0.074 \\
\hline
\end{tabular}
\caption{Certified jamming comparison across models. All three runs exhibit nontrivial certification with zero violations. Gemma achieves full certification coverage on the analysed subset and the tightest certified slack distribution, while the strong positive correlation between jamming index and projected/certified energy seen in the baseline Llama run does not persist in Qwen or Gemma.}
\label{tab:model_generalisation_jamming}
\end{table*}

The certified jamming analysis remains nontrivial under both model substitutions. As shown in Table~\ref{tab:model_generalisation_jamming}, Qwen yields positive certified bounds on a slightly larger fraction of analysed clusters than the baseline Llama run, while Gemma certifies all $40$ analysed clusters. Moreover, Gemma exhibits the tightest certified slack distribution among the three models, with substantially smaller median, mean, and maximum certified slack. At the same time, the empirical relationship between jamming index and interaction energy changes markedly across models. Table~\ref{tab:model_generalisation_jamming} shows that in the Llama baseline, the jamming index is positively correlated with both full projected energy and certified energy, whereas in Qwen the correlations are near zero and in Gemma they are slightly negative. This suggests that certification itself is robust across models, but the empirical coupling between local jamming and projected interaction energy is not architecture-invariant and may reflect model-specific representational organisation.

\subsection{Shearing comparison}
\label{app:model_generalisation_shearing}

Table~\ref{tab:model_generalisation_shearing} summarises the shearing results for the three model families. In addition to the shearing score itself, it reports the slack distribution, which is useful for understanding whether the lower bound is typically tight or loose in each model.

\begin{table*}[t]
\centering
\small
\begin{tabular}{lccc}
\hline
\textbf{Metric} & \textbf{Llama 3.2 3B Instruct} & \textbf{Qwen 2.5 3B Instruct} & \textbf{Gemma 2 2B IT} \\
\hline
Usable edges & 173 & 170 & 191 \\
Slack min & 2.314 & 2.644 & 3.290 \\
Slack median & 7.802 & 11.055 & 9.541 \\
Slack mean & 14.634 & 39652.777 & 28.476 \\
Slack max & 399.525 & 5368921.276 & 2073.852 \\
$D_{\mathrm{shear}}$ min & 0.0354 & 0.0669 & 0.0415 \\
$D_{\mathrm{shear}}$ median & 0.2080 & 0.2393 & 0.1731 \\
$D_{\mathrm{shear}}$ mean & 0.2015 & 0.2355 & 0.1674 \\
$D_{\mathrm{shear}}$ max & 0.3743 & 0.4045 & 0.3668 \\
\hline
\end{tabular}
\caption{Shearing comparison across models. All runs satisfy the lower-bound check, but the slack distributions differ substantially. Qwen exhibits the strongest typical shearing together with the heaviest-tailed slack distribution, whereas Gemma shows the weakest typical shearing of the three models.}
\label{tab:model_generalisation_shearing}
\end{table*}

The shearing experiment again transfers successfully to both additional models. As reported in Table~\ref{tab:model_generalisation_shearing}, the number of usable edges is comparable across all three runs, indicating similar coverage at the edge-analysis stage. The median shearing score $D_{\mathrm{shear}}$ is largest for Qwen, intermediate for the Llama baseline, and smallest for Gemma, suggesting a model-dependent ordering of local transport mismatch under the same atlas construction. The slack distributions, however, are more heterogeneous. Table~\ref{tab:model_generalisation_shearing} shows that Qwen exhibits a markedly heavy-tailed slack distribution, so its raw mean slack is dominated by a small number of extreme edges and is therefore not representative of typical behaviour. Importantly, these large slack values do not violate the lower-bound claim; they correspond to cases in which the bound is valid but loose. For this reason, median slack is the more reliable summary for comparing typical behaviour across models. On that metric, Qwen has the largest typical slack, Gemma is intermediate, and the Llama baseline is smallest.

\subsection{Summary}
\label{app:model_generalisation_summary}

Overall, the model-generalisation experiments show that the local-atlas superposition pipeline is not restricted to a single Llama checkpoint. As summarised in Tables~\ref{tab:model_generalisation_holonomy}, \ref{tab:model_generalisation_jamming}, and~\ref{tab:model_generalisation_shearing}, the full analysis transfers successfully to both Qwen 2.5 3B Instruct and Gemma 2 2B IT, yielding meaningful, nondegenerate outputs for holonomy, shearing, and certified jamming in all cases. At the same time, the resulting statistics are not identical across models. Qwen exhibits the strongest average holonomy and the strongest typical shearing, together with a substantially heavier-tailed shearing slack distribution. Gemma, by contrast, remains close to the Llama baseline in holonomy, exhibits the weakest typical shearing of the three models, and achieves the broadest and tightest certified jamming coverage. Finally, the strong positive relationship between jamming index and projected interaction energy observed in the Llama baseline does not persist in Qwen or Gemma. Overall, these additional experiments support the view that the framework captures genuine model-dependent geometric structure rather than merely reproducing artefacts of one specific checkpoint.

\section{Additional layerwise generalisation experiment on the baseline Llama model}
\label{app:layerwise_generalisation}

To test whether the proposed atlas-based pipeline is tied to a single depth within the baseline model, we repeated the full analysis across multiple hidden layers of \texttt{meta-llama/Llama-3.2-3B-Instruct} while keeping the remainder of the configuration fixed. We evaluated layers $\ell \in \{4,8,12,16,20,24\}$ using the same three data strata (WikiText-103, C4-derived web text, and \texttt{the-stack-smol}), the same activation sampling budget ($200{,}000$ tokens), the same atlas size ($C=128$ clusters), the same local support dimension ($k=32$), and the same downstream settings for transport estimation, holonomy, shearing, and certified jamming. This experiment therefore isolates the effect of representational depth while holding the rest of the atlas construction constant.

All six layerwise runs completed successfully and produced nondegenerate outputs for the three principal diagnostics. In every case, the gauge sanity check continued to hold: tree-edge residuals stayed at numerical precision, while chord residuals and explicit holonomy defects agreed to within very small numerical error. The shearing analysis also remained active at every tested layer, with substantial numbers of usable edges and nonzero transport-mismatch scores throughout. Finally, the certified jamming analysis remained nontrivial at all depths, with positive certified bounds on most analysed clusters and no violations of the certified lower-bound inequality. Taken together, these runs show that the framework is not confined to a single layer of the baseline model, and instead yields a structured depth-dependent profile.

\subsection{Experimental configuration}
\label{app:layerwise_setup}

Table~\ref{tab:layerwise_setup} summarises the fixed configuration used for the layerwise sweep. The only quantity that changes across runs is the hidden layer index; all other atlas, transport, and jamming settings are held constant.

\begin{table}[h]
\centering
\small
\begin{tabular}{ll}
\hline
\textbf{Setting} & \textbf{Value} \\
\hline
Model & \texttt{meta-llama/Llama-3.2-3B-Instruct} \\
Layers tested & $4,8,12,16,20,24$ \\
Activation budget & $200{,}000$ tokens \\
Sequence length & $256$ \\
Sampling stride & every $8$ tokens \\
Number of clusters & $128$ \\
Local basis dimension & $k=32$ \\
kNN graph degree & $6$ \\
Ridge parameter & $\lambda = 10^{-2}$ \\
Minimum overlap & $256$ \\
Maximum overlap & $8000$ \\
Gradient samples per cluster & $512$ \\
Jamming analysis clusters & $40$ \\
Jamming dictionary width & $m=256$ \\
Jamming sparsity parameter & $\alpha = 1.0$ \\
\hline
\end{tabular}
\caption{Configuration used for the layerwise generalisation experiment on the baseline Llama model. All settings were held fixed across runs except for the hidden layer index.}
\label{tab:layerwise_setup}
\end{table}

\subsection{Layerwise summary}
\label{app:layerwise_summary_table}

Table~\ref{tab:layerwise_summary} provides a compact layer-by-layer summary of the main quantities used in the analysis. It collects, in one place, the graph-level holonomy statistics, the certified jamming summaries, and the principal shearing metrics, and serves as the reference point for the more detailed interpretation in the following subsections.

\begin{table*}[t]
\centering
\small
\resizebox{\textwidth}{!}{
\begin{tabular}{lcccccc}
\hline
\textbf{Metric} & \textbf{Layer 4} & \textbf{Layer 8} & \textbf{Layer 12} & \textbf{Layer 16} & \textbf{Layer 20} & \textbf{Layer 24} \\
\hline
Connected components & 9 & 6 & 9 & 7 & 11 & 8 \\
Largest connected component size & 40 & 95 & 70 & 90 & 81 & 51 \\
Tree residual mean & $1.35\times 10^{-5}$ & $1.93\times 10^{-5}$ & $1.93\times 10^{-5}$ & $1.60\times 10^{-5}$ & $1.53\times 10^{-5}$ & $1.16\times 10^{-5}$ \\
Chord edges & 24 & 76 & 101 & 84 & 85 & 80 \\
Normalised holonomy mean $D_{\mathrm{hol}}$ & 0.668 & 0.578 & 0.430 & 0.554 & 0.607 & 0.588 \\
Normalised holonomy max $D_{\mathrm{hol}}$ & 0.841 & 0.915 & 0.780 & 0.842 & 0.890 & 0.794 \\
Saved clusters for jamming & 107 & 122 & 113 & 124 & 121 & 111 \\
Clusters with positive certified bound & 39 & 39 & 36 & 35 & 36 & 36 \\
Certified slack median & 6.176 & 7.582 & 8.300 & 7.158 & 7.114 & 5.996 \\
$\mathrm{Corr}(J,E_{\mathrm{full\_proj}})$ & 0.617 & 0.806 & 0.304 & 0.634 & 0.795 & 0.073 \\
$\mathrm{Corr}(J,E_{\mathrm{cert}})$ & 0.627 & 0.843 & 0.302 & 0.539 & 0.779 & 0.077 \\
Usable edges in shearing analysis & 80 & 177 & 178 & 183 & 176 & 135 \\
Slack median & 15.547 & 6.343 & 6.419 & 8.119 & 7.540 & 7.750 \\
$D_{\mathrm{shear}}$ median & 0.262 & 0.214 & 0.162 & 0.206 & 0.208 & 0.225 \\
\hline
\end{tabular}}
\caption{Layerwise summary of holonomy, certified jamming, and shearing on \texttt{meta-llama/Llama-3.2-3B-Instruct}. All quantities are computed under a matched atlas configuration; only the hidden layer changes.}
\label{tab:layerwise_summary}
\end{table*}

\subsection{Holonomy across depth}
\label{app:layerwise_holonomy}

As shown in Table~\ref{tab:layerwise_summary}, the holonomy analysis remains valid at every tested depth. Across all six layers, the mean tree residual stays on the order of $10^{-5}$, indicating that the spanning-tree gauge construction remains numerically stable throughout the sweep. Moreover, the chord residuals and explicit holonomy defects continue to agree to numerical precision at every layer, so the cycle-level gauge computation is not restricted to one special depth.

At the same time, the magnitude of the holonomy signal is clearly depth-dependent. Table~\ref{tab:layerwise_summary} shows that the mean normalised holonomy $D_{\mathrm{hol}}$ is relatively high at layer $4$ ($0.668$), decreases through layers $8$ and $12$ to a minimum at layer $12$ ($0.430$), and then rises again in the deeper portion of the network, reaching $0.607$ at layer $20$ and remaining elevated at layer $24$ ($0.588$). Thus, the holonomy profile is non-monotonic: cycle inconsistency is strongest early, weakest around the middle of the sweep, and then re-emerges at later layers. This is already evidence that the geometric obstruction measured by the holonomy diagnostic varies systematically with representational depth rather than appearing as a flat artefact of the pipeline.

\subsection{Certified jamming across depth}
\label{app:layerwise_jamming}

The certified jamming analysis also remains active throughout the layer sweep. As reported in Table~\ref{tab:layerwise_summary}, positive certified bounds are obtained for the majority of analysed clusters at every layer, ranging from $35/40$ to $39/40$, and there are no certification violations. This shows that the local sparse-coding and certification stage is not confined to one part of the model.

However, the empirical relation between jamming and energy varies considerably with depth. Table~\ref{tab:layerwise_summary} shows that the correlation $\mathrm{Corr}(J,E_{\mathrm{full\_proj}})$ is moderately strong at layer $4$ ($0.617$), becomes very strong at layer $8$ ($0.806$), weakens markedly at layer $12$ ($0.304$), rises again at layers $16$ and $20$ ($0.634$ and $0.795$ respectively), and then nearly disappears by layer $24$ ($0.073$). The same pattern appears for the correlation between $J$ and certified energy on the certified subset. Thus, while certification itself is broadly stable across depth, the extent to which the jamming index empirically tracks projected interaction energy is highly layer-dependent. This suggests that the local interference structure measured by $J$ is not equally predictive of projected energy at all stages of the network.

The certified slack distribution also changes with depth, though less dramatically than the correlation structure. As seen in Table~\ref{tab:layerwise_summary}, the median certified slack ranges from roughly $6.0$ to $8.3$, with the largest median value observed at layer $12$. In other words, certification remains present across the entire sweep, but its tightness and empirical interpretation vary with representational depth.

\subsection{Shearing across depth}
\label{app:layerwise_shearing}

The shearing analysis produces usable edge sets and nonzero transport-mismatch scores at every tested layer. Table~\ref{tab:layerwise_summary} shows that the median shearing score $D_{\mathrm{shear}}$ exhibits a particularly clear depth profile: it is largest at layer $4$ ($0.262$), decreases through layers $8$ and $12$ to a minimum at layer $12$ ($0.162$), and then increases again in later layers, reaching $0.225$ by layer $24$. This non-monotonic pattern closely mirrors the broad structure observed in the holonomy measurements.

The slack distribution is again more variable than the shearing score itself. Table~\ref{tab:layerwise_summary} shows that layer $4$ has the largest median slack ($15.547$), whereas the middle and later layers fall into a lower band between approximately $6.3$ and $8.1$. As in the model-comparison experiments, large slack values do not indicate a violation of the lower-bound claim; rather, they correspond to cases in which the bound is valid but relatively loose. For this reason, median slack is more informative than mean slack when comparing typical behaviour across layers.

Taken together, the shearing results indicate that local transport mismatch is strongest at the earliest tested layer, weakest near the middle of the sweep, and then increases again toward later layers. This reinforces the broader conclusion that the geometric obstruction measured by the atlas construction is structured across depth rather than uniformly distributed.

\subsection{Summary}
\label{app:layerwise_summary_text}

Overall, the layerwise experiment shows that the proposed framework is not only operational across models but also sensitive to representational depth within a single model. As summarised in Table~\ref{tab:layerwise_summary}, holonomy, shearing, and certified jamming all remain measurable and nontrivial at every tested layer of \texttt{meta-llama/Llama-3.2-3B-Instruct}, but none of these diagnostics is constant across depth. Instead, the results exhibit a structured non-monotonic profile: both holonomy and shearing are comparatively strong at layer $4$, weaken around layer $12$, and then rise again in later layers. Certified jamming persists throughout the sweep, while the empirical correlation between jamming index and projected interaction energy varies strongly with depth, becoming especially weak at layers $12$ and $24$. These observations support the view that the framework is capturing genuine depth-dependent geometric structure rather than simply returning the same statistic at every layer.

\section{Additional reproducibility experiment: seed stability on the baseline Llama model}
\label{app:seed_reproducibility}

To test whether the main empirical findings are sensitive to random initialisation in the downstream stages of the pipeline, we performed a seed-reproducibility sweep on the baseline model \texttt{meta-llama/Llama-3.2-3B-Instruct} at layer $\ell=16$. In this experiment we held the baseline activation set and cluster graph fixed and varied only the seed-aware downstream stages, namely local basis estimation, edge-transport estimation, cluster-wise activation/gradient sampling, and the certified jamming computation. We evaluated seeds $s \in \{0,1,2,3,4\}$ under the same atlas configuration used in the main experiments. This design isolates reproducibility of the geometric and interference diagnostics under repeated stochastic downstream fitting, without conflating that variation with changes in the extracted activation corpus or the upstream clustering.

All five seed runs completed successfully and produced nondegenerate outputs for holonomy, shearing, and certified jamming. Across all seeds, the gauge sanity check remained stable: tree-edge residuals stayed at numerical precision, the number of chord edges in the largest connected component remained fixed, and chord residuals matched holonomy defects to numerical precision. The shearing and jamming stages also remained qualitatively stable across all seeds, with closely matched median shearing scores, closely matched certified slack statistics, and a consistently positive correlation between the jamming index and projected interaction energy. Overall, this experiment shows that the core findings for the baseline Llama run are not artefacts of a single downstream seed.

\subsection{Experimental configuration}
\label{app:seed_reproducibility_setup}

Table~\ref{tab:seed_reproducibility_setup} summarises the configuration used for the reproducibility sweep. The important point is that this is a \emph{downstream} seed experiment: the baseline activations and cluster graph are fixed, and only the seed-sensitive later stages are rerun.

\begin{table}[h]
\centering
\small
\begin{tabular}{ll}
\hline
\textbf{Setting} & \textbf{Value} \\
\hline
Model & \texttt{meta-llama/Llama-3.2-3B-Instruct} \\
Layer & $16$ \\
Seeds tested & $0,1,2,3,4$ \\
Activation source & fixed baseline activations \\
Graph source & fixed baseline cluster graph \\
Number of clusters & $128$ \\
Local basis dimension & $k=32$ \\
kNN graph degree & $6$ \\
Ridge parameter & $\lambda = 10^{-2}$ \\
Minimum overlap & $256$ \\
Maximum overlap & $8000$ \\
Gradient samples per cluster & $512$ \\
Jamming analysis clusters & $40$ \\
Jamming dictionary width & $m=256$ \\
Jamming sparsity parameter & $\alpha = 1.0$ \\
\hline
\end{tabular}
\caption{Configuration used for the downstream seed-reproducibility experiment. The baseline activations and cluster graph were held fixed, while seed-aware downstream stages were rerun for each seed.}
\label{tab:seed_reproducibility_setup}
\end{table}

\subsection{Seedwise summary}
\label{app:seed_reproducibility_table}

Table~\ref{tab:seed_reproducibility} gives the full seed-by-seed summary of the main quantities reported by the pipeline. The most important features to note are the near-constancy of the graph-level holonomy quantities, the narrow variation in jamming certification statistics, and the very small spread in the median shearing score across seeds.

\begin{table*}[t]
\centering
\small
\begin{tabular}{lccccc}
\hline
\textbf{Metric} & \textbf{Seed 0} & \textbf{Seed 1} & \textbf{Seed 2} & \textbf{Seed 3} & \textbf{Seed 4} \\
\hline
Largest connected component size & 90 & 90 & 90 & 90 & 90 \\
Chord edges & 84 & 84 & 84 & 84 & 84 \\
Tree residual mean & $1.60\times 10^{-5}$ & $1.60\times 10^{-5}$ & $1.56\times 10^{-5}$ & $1.57\times 10^{-5}$ & $1.57\times 10^{-5}$ \\
Normalised holonomy mean $D_{\mathrm{hol}}$ & 0.5541 & 0.5445 & 0.5582 & 0.5556 & 0.5453 \\
Normalised holonomy max $D_{\mathrm{hol}}$ & 0.8420 & 0.8546 & 0.8851 & 0.8925 & 0.8566 \\
Saved clusters for jamming & 124 & 124 & 124 & 124 & 124 \\
Clusters with positive certified bound & 35 & 34 & 34 & 35 & 35 \\
Certified slack median & 7.158 & 7.580 & 7.400 & 7.277 & 8.157 \\
$\mathrm{Corr}(J,E_{\mathrm{full\_proj}})$ & 0.634 & 0.601 & 0.624 & 0.624 & 0.607 \\
$\mathrm{Corr}(J,E_{\mathrm{cert}})$ & 0.539 & 0.575 & 0.630 & 0.592 & 0.534 \\
Usable edges in shearing analysis & 183 & 183 & 183 & 183 & 183 \\
Slack median & 8.119 & 8.283 & 8.245 & 7.928 & 8.415 \\
$D_{\mathrm{shear}}$ median & 0.2055 & 0.2058 & 0.2100 & 0.2041 & 0.2032 \\
\hline
\end{tabular}
\caption{Seedwise summary of the downstream reproducibility experiment on \texttt{meta-llama/Llama-3.2-3B-Instruct} at layer $16$. The baseline activations and graph were fixed; only seed-aware downstream stages were rerun.}
\label{tab:seed_reproducibility}
\end{table*}

\subsection{Compact stability summary}
\label{app:seed_reproducibility_compact}

While Table~\ref{tab:seed_reproducibility} shows the full seedwise results, Table~\ref{tab:seed_reproducibility_compact} condenses the same information into mean/std summaries across seeds. This makes the small magnitude of the run-to-run variation easier to see directly. In particular, the standard deviations for the main quantities are all modest, with especially tight concentration for the mean holonomy and the median shearing score.

\begin{table}[h]
\centering
\small
\begin{tabular}{lcc}
\hline
\textbf{Metric} & \textbf{Mean across seeds} & \textbf{Std.\ across seeds} \\
\hline
Normalised holonomy mean $D_{\mathrm{hol}}$ & 0.5515 & 0.0061 \\
Clusters with positive certified bound & 34.6 & 0.5 \\
Certified slack median & 7.514 & 0.365 \\
$\mathrm{Corr}(J,E_{\mathrm{full\_proj}})$ & 0.618 & 0.012 \\
$\mathrm{Corr}(J,E_{\mathrm{cert}})$ & 0.574 & 0.036 \\
Slack median & 8.198 & 0.170 \\
$D_{\mathrm{shear}}$ median & 0.2057 & 0.0023 \\
\hline
\end{tabular}
\caption{Compact mean/std summary across seeds for the main reproducibility metrics. Variability is small for all reported quantities, indicating strong downstream seed stability.}
\label{tab:seed_reproducibility_compact}
\end{table}

\subsection{Holonomy stability across seeds}
\label{app:seed_reproducibility_holonomy}

The holonomy computation is highly stable across seeds. As shown in Tables~\ref{tab:seed_reproducibility} and~\ref{tab:seed_reproducibility_compact}, the largest connected component size remains fixed at $90$ in every run, the number of chord edges remains fixed at $84$, and the mean tree residual stays tightly concentrated around $1.6\times 10^{-5}$. Likewise, the mean normalised holonomy $D_{\mathrm{hol}}$ varies only slightly across the sweep, from $0.5445$ to $0.5582$, with a standard deviation of approximately $0.006$. This shows that the baseline holonomy signal is not a one-seed artefact of local basis estimation or downstream fitting. Rather, the gauge-fixing and cycle-defect measurements are highly reproducible under the seed perturbations considered here.

\subsection{Certified jamming stability across seeds}
\label{app:seed_reproducibility_jamming}

The certified jamming results are also stable across seeds. Table~\ref{tab:seed_reproducibility} shows that the number of analysed clusters remains fixed at $40$ in all runs, while the number of clusters with positive certified bound varies only between $34$ and $35$. The median certified slack lies in a narrow interval between approximately $7.28$ and $8.16$, and Table~\ref{tab:seed_reproducibility_compact} shows that its standard deviation across seeds is only $0.365$. Most importantly, the correlation between the jamming index and full projected energy remains consistently positive across all seeds, ranging from $0.601$ to $0.634$, while the corresponding correlation on the certified subset remains positive as well. Thus, the empirical coupling between local jamming and projected interaction energy in the baseline run is not fragile to downstream reseeding.

\subsection{Shearing stability across seeds}
\label{app:seed_reproducibility_shearing}

The shearing results exhibit similarly strong reproducibility. As reported in Table~\ref{tab:seed_reproducibility}, the number of usable edges in the shearing analysis remains fixed at $183$ across all seeds. The median shearing score $D_{\mathrm{shear}}$ lies in the very narrow range $[0.2032,\,0.2100]$, and Table~\ref{tab:seed_reproducibility_compact} shows a standard deviation of only $0.0023$. The median slack is likewise stable, lying between $7.93$ and $8.42$. These results indicate that the typical strength of the shearing signal is essentially unchanged across downstream seeds, even though the individual local fits are stochastic.

\subsection{Summary}
\label{app:seed_reproducibility_summary}

Taken together, these seed sweeps show that the baseline Llama layer-$16$ findings are robust to randomness in the downstream seed-sensitive stages of the pipeline. The holonomy signal, the shearing signal, and the certified jamming signal all remain qualitatively unchanged across the five seeds tested. In particular, the gauge sanity checks remain numerically valid, the magnitudes of the principal diagnostics vary only modestly, and the positive empirical relation between jamming index and projected interaction energy persists throughout the sweep. While this experiment does not constitute a full end-to-end reseeding of data extraction and clustering, it provides strong evidence that the core results are not artifacts of a single downstream random seed.

\section{Additional null-control experiment: random local bases}
\label{app:null_random_bases}

To test whether the reported geometric diagnostics depend on the \emph{learned} local chart structure rather than merely on the existence of arbitrary orthonormal frames, we performed a random-bases null control on the baseline \texttt{meta-llama/Llama-3.2-3B-Instruct} run at layer $\ell=16$. In this null experiment, we held the baseline activations, cluster assignments, and graph fixed, but replaced the learned PCA-based local bases in each cluster with random orthonormal bases of the same shape. We then reran the downstream edge-transport, holonomy, and shearing analyses without changing any other parameters. This null therefore preserves the combinatorial atlas structure while destroying the learned local alignment that the method is intended to capture.

Table~\ref{tab:null_random_bases} summarises the comparison between the baseline run and the random-bases null. Several important points emerge. First, the numerical gauge construction itself remains well behaved under the null: the largest connected component size remains $90$, the number of chord edges remains $84$, and the tree residual remains at numerical precision. Thus, the null does not simply break the pipeline mechanically. Second, however, the geometric quantities change dramatically. The mean normalised holonomy increases from $0.554$ in the baseline run to $1.002$ under random bases, while the median shearing score $D_{\mathrm{shear}}$ increases from $0.206$ to $0.644$. In other words, once the learned local bases are replaced by arbitrary orthonormal frames, both cycle inconsistency and local transport mismatch become substantially larger.

Interestingly, the slack distribution becomes \emph{less} heavy-tailed under the null, with the median slack decreasing from $8.119$ to $3.932$ and the maximum slack dropping sharply. This does not contradict the interpretation above. Slack measures the looseness of the lower bound, not the absolute quality of the local geometry. In the random-bases null, the realised mismatch becomes much larger, so the ratio between realised mismatch and lower bound can shrink even though the geometric structure itself becomes much less meaningful. For this reason, the decisive quantities in this null control are the large increases in holonomy and shearing, not the decrease in slack.

Overall, this null control strongly supports the interpretation that the framework is sensitive to learned local geometry rather than arbitrary frame choices. If the local bases were irrelevant, replacing them with random orthonormal charts would not materially change the results. Instead, the experiment shows the opposite: the downstream computations still run, but the resulting geometry becomes substantially more inconsistent. This indicates that the reported baseline signals are not generic artefacts of any orthonormal charting of the latent space, but depend on the specific learned local alignment extracted from the activations.

\begin{table*}[t]
\centering
\small
\begin{tabular}{lcc}
\hline
\textbf{Metric} & \textbf{Baseline Llama} & \textbf{Random-bases null} \\
\hline
Largest connected component size & 90 & 90 \\
Chord edges & 84 & 84 \\
Tree residual mean & $1.60\times 10^{-5}$ & $1.51\times 10^{-5}$ \\
Normalised holonomy mean $D_{\mathrm{hol}}$ & 0.554 & 1.002 \\
Normalised holonomy max $D_{\mathrm{hol}}$ & 0.842 & 1.045 \\
Usable edges in shearing analysis & 183 & 183 \\
Slack min & 2.314 & 2.009 \\
Slack median & 8.119 & 3.932 \\
Slack mean & 14.886 & 4.937 \\
Slack max & 399.525 & 41.516 \\
$D_{\mathrm{shear}}$ min & 0.0354 & 0.6062 \\
$D_{\mathrm{shear}}$ median & 0.2055 & 0.6437 \\
$D_{\mathrm{shear}}$ mean & 0.1989 & 0.6437 \\
$D_{\mathrm{shear}}$ max & 0.3743 & 0.6942 \\
\hline
\end{tabular}
\caption{Comparison between the baseline Llama run and a null control in which the learned local bases are replaced by random orthonormal bases of the same shape. The numerical gauge construction remains stable, but both holonomy and shearing increase sharply under the null, indicating that the reported geometric structure depends on learned local alignment rather than arbitrary frame choices.}
\label{tab:null_random_bases}
\end{table*}

\section{Additional hyperparameter-robustness experiment: atlas granularity and local basis dimension}
\label{app:hparam_robustness}

To test whether the main empirical findings depend on a single atlas configuration, we performed a targeted hyperparameter-robustness sweep on the baseline model \texttt{meta-llama/Llama-3.2-3B-Instruct} at layer $\ell=16$. In this experiment we held the model, layer, input corpora, activation budget, and all downstream solver settings fixed, and varied only two atlas-construction parameters: the number of clusters $C$ and the local basis dimension $k$. We evaluated the configurations $(C,k)\in\{(64,16),(64,32),(128,16),(128,32),(256,16),(256,32)\}$, where $(128,32)$ is the baseline setting used in the main text. This sweep probes both atlas granularity and local chart dimension while preserving the rest of the pipeline.

All tested configurations completed successfully and produced nondegenerate outputs for holonomy, shearing, and certified jamming. In every case, the gauge sanity check remained numerically valid: tree-edge residuals stayed at numerical precision and chord residuals matched explicit holonomy defects to within small numerical error. Moreover, all runs produced nonzero shearing scores and nontrivial certified jamming results with no certification violations. Thus, the framework is not restricted to a single hand-picked atlas setting.

\subsection{Experimental configuration}
\label{app:hparam_robustness_setup}

Table~\ref{tab:hparam_robustness_setup} summarises the configuration of the sweep. Only the atlas granularity $C$ and the local basis dimension $k$ were varied.

\begin{table}[h]
\centering
\small
\begin{tabular}{ll}
\hline
\textbf{Setting} & \textbf{Value} \\
\hline
Model & \texttt{meta-llama/Llama-3.2-3B-Instruct} \\
Layer & $16$ \\
Atlas sizes tested & $C\in\{64,128,256\}$ \\
Local basis dimensions tested & $k\in\{16,32\}$ \\
Activation budget & $200{,}000$ tokens \\
Sequence length & $256$ \\
Sampling stride & every $8$ tokens \\
kNN graph degree & $6$ \\
Ridge parameter & $\lambda = 10^{-2}$ \\
Minimum overlap & $256$ \\
Maximum overlap & $8000$ \\
Gradient samples per cluster & $512$ \\
Jamming analysis clusters & $40$ \\
Jamming dictionary width & $m=256$ \\
Jamming sparsity parameter & $\alpha = 1.0$ \\
\hline
\end{tabular}
\caption{Configuration used for the atlas hyperparameter-robustness experiment. All settings were fixed except the number of clusters $C$ and the local basis dimension $k$.}
\label{tab:hparam_robustness_setup}
\end{table}

\subsection{Configuration-by-configuration summary}
\label{app:hparam_robustness_table}

Table~\ref{tab:hparam_robustness} reports the main summary statistics for each tested $(C,k)$ pair. The most important patterns are the relative stability of mean holonomy across the full sweep, the gradual decrease in typical shearing as the atlas becomes finer, and the persistence of certified jamming across all configurations.

\begin{table*}[t]
\centering
\small
\resizebox{\textwidth}{!}{
\begin{tabular}{lcccccc}
\hline
\textbf{Metric} & \textbf{$(64,16)$} & \textbf{$(64,32)$} & \textbf{$(128,16)$} & \textbf{$(128,32)$} & \textbf{$(256,16)$} & \textbf{$(256,32)$} \\
\hline
Graph edges & 310 & 310 & 580 & 580 & 1188 & 1188 \\
Saved edge ops & 142 & 142 & 192 & 192 & 171 & 171 \\
Connected components & 2 & 2 & 7 & 7 & 28 & 28 \\
Largest connected component size & 55 & 55 & 90 & 90 & 99 & 99 \\
Chord edges in LCC & 85 & 85 & 84 & 84 & 23 & 23 \\
Tree residual mean & $7.24\times 10^{-6}$ & $1.26\times 10^{-5}$ & $9.25\times 10^{-6}$ & $1.60\times 10^{-5}$ & $1.60\times 10^{-5}$ & $2.44\times 10^{-5}$ \\
Normalised holonomy mean $D_{\mathrm{hol}}$ & 0.576 & 0.559 & 0.561 & 0.554 & 0.566 & 0.537 \\
Normalised holonomy max $D_{\mathrm{hol}}$ & 0.874 & 0.765 & 0.899 & 0.842 & 0.749 & 0.702 \\
Saved clusters for jamming & 64 & 64 & 124 & 124 & 237 & 237 \\
Clusters with positive certified bound & 36 & 36 & 35 & 35 & 38 & 38 \\
Certified slack median & 5.791 & 5.791 & 7.158 & 7.158 & 6.453 & 6.453 \\
$\mathrm{Corr}(J,E_{\mathrm{full\_proj}})$ & 0.695 & 0.695 & 0.634 & 0.634 & 0.376 & 0.376 \\
$\mathrm{Corr}(J,E_{\mathrm{cert}})$ & 0.538 & 0.538 & 0.539 & 0.539 & 0.419 & 0.419 \\
Usable edges in shearing analysis & 139 & 139 & 183 & 183 & 162 & 162 \\
Slack median & 5.750 & 6.810 & 6.431 & 8.119 & 6.148 & 6.955 \\
$D_{\mathrm{shear}}$ median & 0.245 & 0.218 & 0.206 & 0.206 & 0.186 & 0.193 \\
\hline
\end{tabular}}
\caption{Hyperparameter-robustness sweep over atlas size $C$ and local basis dimension $k$ on \texttt{meta-llama/Llama-3.2-3B-Instruct} at layer $16$. The baseline setting used in the main text is $(C,k)=(128,32)$.}
\label{tab:hparam_robustness}
\end{table*}

\subsection{Holonomy robustness}
\label{app:hparam_robustness_holonomy}

The holonomy signal remains clearly present across all tested atlas configurations. As shown in Table~\ref{tab:hparam_robustness}, the mean normalised holonomy $D_{\mathrm{hol}}$ stays within a relatively narrow band, ranging from $0.537$ to $0.576$ across the full sweep. This is a notably small variation given that the atlas size changes by a factor of four and the local chart dimension doubles. In other words, the presence of nontrivial cycle inconsistency is highly robust to the tested atlas choices.

At the same time, the graph structure changes substantially with atlas granularity. Coarser atlases ($C=64$) yield fewer total graph edges but relatively well-connected largest components, whereas finer atlases ($C=256$) produce many more graph edges overall but far fewer chord edges in the largest connected component after overlap filtering. Thus, the holonomy computation remains meaningful across all tested atlas granularities, but the effective cycle structure on which it is computed becomes sparser as the atlas becomes finer. This is an interpretable and expected behaviour rather than a failure mode.

\subsection{Certified jamming robustness}
\label{app:hparam_robustness_jamming}

Certified jamming also remains nontrivial throughout the sweep. Every tested configuration yields positive certified bounds on most of the $40$ analysed clusters, with certification counts ranging from $35$ to $38$, and no certification violations. This indicates that the local interference structure identified by the certification procedure is not tied to a single atlas resolution.

The most interesting variation appears in the correlation between the jamming index and projected interaction energy. At the coarsest atlas setting ($C=64$), $\mathrm{Corr}(J,E_{\mathrm{full\_proj}})$ is relatively strong ($0.695$). At the baseline resolution ($C=128$), it remains clearly positive ($0.634$). At the finest atlas setting ($C=256$), however, it drops to $0.376$. The same broad weakening is visible in the certified-energy correlation. This suggests that finer atlases may separate the representation into more localised clusters in a way that weakens the direct coupling between the jamming index and projected interaction energy, even though certification itself remains present.

It is also worth noting that the jamming summaries are unchanged when only $k$ varies at fixed $C$. This is expected from the current pipeline design: the jamming stage depends on the cluster-level activation/gradient samples and therefore tracks the choice of atlas partition $C$, but does not directly depend on the local basis dimension used in the transport computations.

\subsection{Shearing robustness}
\label{app:hparam_robustness_shearing}

The shearing analysis remains active across all tested settings, and the resulting pattern is particularly interpretable. The median shearing score $D_{\mathrm{shear}}$ is largest for the coarsest atlas, taking values $0.245$ and $0.218$ for $(64,16)$ and $(64,32)$ respectively. It is intermediate near the baseline setting, at approximately $0.206$, and becomes smaller for the finest atlas, taking values $0.186$ and $0.193$ for $(256,16)$ and $(256,32)$. Thus, typical local transport mismatch decreases as the atlas becomes finer.

This trend is intuitively reasonable: a finer atlas provides smaller local patches and can therefore reduce the amount of mismatch that must be absorbed by a single edge transport. By contrast, a coarser atlas forces broader local variation into fewer charts, which increases the typical shearing burden on edge transports. The effect of $k$ at fixed $C$ is more modest than the effect of changing $C$, indicating that atlas granularity is the dominant factor in this sweep.

The slack statistics vary more than the shearing scores themselves, but remain finite and nontrivial across all runs. As elsewhere in the paper, median slack is the more informative summary for typical behaviour than mean slack, because the mean can be influenced by rare high-slack edges.

\subsection{Summary}
\label{app:hparam_robustness_summary}

Taken together, these hyperparameter sweeps show that the framework is robust to substantial changes in atlas granularity and local basis dimension. Across all tested $(C,k)$ settings, the gauge sanity checks remain valid, holonomy remains nonzero, shearing remains nonzero, and certified jamming remains nontrivial. The qualitative presence of the three main diagnostics therefore does not depend on a single hand-chosen atlas configuration.

At the same time, the magnitudes vary in structured and interpretable ways. Mean holonomy is notably stable across the sweep, suggesting that cycle inconsistency is a robust feature of the representation under the tested atlas choices. Typical shearing decreases as the atlas becomes finer, indicating that transport mismatch is sensitive to chart granularity. Certified jamming persists throughout, but the empirical coupling between jamming and projected energy weakens for finer atlases. Overall, this experiment supports the view that the framework is robust in kind while remaining meaningfully sensitive to how the atlas is constructed.

\section{Additional graph-construction robustness experiment}
\label{app:graph_construction_robustness}

To test whether the principal geometric diagnostics depend strongly on one particular graph-construction rule, we repeated the baseline Llama experiment while varying the neighbourhood parameter used in the cluster graph. Concretely, we kept the model, layer, activation set, cluster count, local basis dimension, and all downstream analysis settings fixed, and varied only the graph degree in the centroid-based $k$-nearest-neighbour graph. In addition to the baseline setting $\mathrm{knn}=6$, we evaluated $\mathrm{knn}\in\{4,8\}$. This experiment therefore isolates the effect of modest graph sparsification or densification while leaving the remainder of the atlas construction unchanged.

Both additional runs completed successfully and produced nondegenerate outputs for holonomy, shearing, and certified jamming. In both cases, the gauge sanity check remained valid: tree-edge residuals stayed at numerical precision, while chord residuals matched explicit holonomy defects to within small numerical error. The resulting comparison shows that moderate changes in the graph-construction rule alter the graph combinatorics in the expected way, but do not materially change the principal geometric conclusions. In particular, both mean holonomy and median shearing remain very close to the baseline values across the tested graph degrees.

\subsection{Experimental configuration}
\label{app:graph_construction_robustness_setup}

Table~\ref{tab:graph_construction_robustness_setup} summarises the fixed configuration used for the graph-construction robustness experiment. The only quantity varied across runs is the graph neighbourhood parameter $\mathrm{knn}$.

\begin{table}[h]
\centering
\small
\begin{tabular}{ll}
\hline
\textbf{Setting} & \textbf{Value} \\
\hline
Model & \texttt{meta-llama/Llama-3.2-3B-Instruct} \\
Layer & $16$ \\
Activation source & fixed baseline activations \\
Atlas size & $C=128$ \\
Local basis dimension & $k=32$ \\
Graph degrees tested & $\mathrm{knn}\in\{4,6,8\}$ \\
Ridge parameter & $\lambda = 10^{-2}$ \\
Minimum overlap & $256$ \\
Maximum overlap & $8000$ \\
Gradient samples per cluster & $512$ \\
Jamming analysis clusters & $40$ \\
Jamming dictionary width & $m=256$ \\
Jamming sparsity parameter & $\alpha = 1.0$ \\
\hline
\end{tabular}
\caption{Configuration used for the graph-construction robustness experiment. All settings were fixed except the neighbourhood parameter $\mathrm{knn}$ used in the cluster graph.}
\label{tab:graph_construction_robustness_setup}
\end{table}

\subsection{Summary of results}
\label{app:graph_construction_robustness_table}

Table~\ref{tab:graph_construction_robustness} summarises the principal graph-level, jamming, and shearing quantities for $\mathrm{knn}=4,6,8$. The most important pattern is that the graph combinatorics change as expected with $\mathrm{knn}$, but the main holonomy and shearing summaries remain highly stable.

\begin{table*}[t]
\centering
\small
\begin{tabular}{lccc}
\hline
\textbf{Metric} & \textbf{$\mathrm{knn}=4$} & \textbf{$\mathrm{knn}=6$ (baseline)} & \textbf{$\mathrm{knn}=8$} \\
\hline
Graph edges & 388 & 580 & 780 \\
Saved edge ops & 177 & 192 & 200 \\
Connected components & 10 & 7 & 7 \\
Largest connected component size & 84 & 90 & 90 \\
Chord edges in LCC & 74 & 84 & 92 \\
Tree residual mean & $1.56\times 10^{-5}$ & $1.60\times 10^{-5}$ & $1.70\times 10^{-5}$ \\
Normalised holonomy mean $D_{\mathrm{hol}}$ & 0.551 & 0.554 & 0.582 \\
Normalised holonomy max $D_{\mathrm{hol}}$ & 0.842 & 0.842 & 0.924 \\
Saved clusters for jamming & 124 & 124 & 124 \\
Clusters with positive certified bound & 35 & 35 & 35 \\
Certified slack median & 7.158 & 7.158 & 7.158 \\
$\mathrm{Corr}(J,E_{\mathrm{full\_proj}})$ & 0.634 & 0.634 & 0.634 \\
$\mathrm{Corr}(J,E_{\mathrm{cert}})$ & 0.539 & 0.539 & 0.539 \\
Usable edges in shearing analysis & 170 & 183 & 191 \\
Slack median & 8.052 & 8.119 & 7.940 \\
$D_{\mathrm{shear}}$ median & 0.203 & 0.206 & 0.208 \\
\hline
\end{tabular}
\caption{Graph-construction robustness sweep over the neighbourhood parameter $\mathrm{knn}$ on \texttt{meta-llama/Llama-3.2-3B-Instruct} at layer $16$. The baseline setting used in the main text is $\mathrm{knn}=6$.}
\label{tab:graph_construction_robustness}
\end{table*}

\subsection{Holonomy robustness under graph changes}
\label{app:graph_construction_robustness_holonomy}

As shown in Table~\ref{tab:graph_construction_robustness}, varying $\mathrm{knn}$ changes the graph structure in the expected direction: $\mathrm{knn}=4$ produces a sparser graph with fewer total edges and fewer chord edges in the largest connected component, whereas $\mathrm{knn}=8$ produces a denser graph with more retained edges and more cycle opportunities. Despite this, the holonomy computation remains numerically stable throughout. The mean tree residual remains on the order of $10^{-5}$ in all three runs, indicating that the gauge-fixing step is insensitive to these modest graph changes.

More importantly, the mean normalised holonomy $D_{\mathrm{hol}}$ remains very close across the sweep: it is $0.551$ for $\mathrm{knn}=4$, $0.554$ for the baseline $\mathrm{knn}=6$, and $0.582$ for $\mathrm{knn}=8$. Thus, while increasing $\mathrm{knn}$ changes the graph combinatorics and slightly increases the available cycle structure, the magnitude of the holonomy signal remains highly stable. This indicates that the observed cycle inconsistency is not an artefact of one arbitrary neighbourhood choice.

\subsection{Certified jamming under graph changes}
\label{app:graph_construction_robustness_jamming}

Table~\ref{tab:graph_construction_robustness} also shows that the jamming summaries are unchanged across the tested graph constructions. This is expected from the structure of the pipeline: the certified jamming stage depends on the cluster-level activation/gradient samples, which are determined by the fixed activations and cluster partition, rather than by the graph neighbourhood rule. Accordingly, the number of certified clusters, the median certified slack, and the correlations between the jamming index and projected/certified energy remain identical across $\mathrm{knn}=4,6,8$.

This invariance is useful rather than problematic. It confirms that changing the graph-construction rule isolates the holonomy and shearing parts of the analysis without inadvertently changing the jamming stage. In other words, the graph sweep is a clean robustness test for the graph-dependent parts of the framework.

\subsection{Shearing robustness under graph changes}
\label{app:graph_construction_robustness_shearing}

The shearing results are similarly stable. As shown in Table~\ref{tab:graph_construction_robustness}, the number of usable edges in the shearing analysis changes with $\mathrm{knn}$, from $170$ at $\mathrm{knn}=4$ to $191$ at $\mathrm{knn}=8$, which is consistent with the denser graph producing more candidate edges. However, the median shearing score changes only minimally: $0.203$ for $\mathrm{knn}=4$, $0.206$ at baseline, and $0.208$ for $\mathrm{knn}=8$. The median slack is likewise nearly unchanged across the sweep.

This is a strong robustness result. It shows that the typical magnitude of the shearing signal does not depend materially on whether the graph is constructed slightly more sparsely or slightly more densely. The local transport mismatch diagnosed by the framework therefore appears to reflect a property of the learned representation itself rather than a fragile consequence of one particular graph recipe.

\subsection{Summary}
\label{app:graph_construction_robustness_summary}

Taken together, these graph-construction experiments show that the main geometric conclusions are robust to moderate changes in the neighbourhood rule used to build the atlas graph. As summarised in Table~\ref{tab:graph_construction_robustness}, varying $\mathrm{knn}$ changes the combinatorial properties of the graph in the expected way, but does not materially change the principal holonomy or shearing summaries. Mean holonomy remains in a narrow range, median shearing remains almost unchanged, and the graph-independent jamming summaries remain exactly stable. These results support the view that the reported geometric structure is not an artefact of one arbitrary graph-construction choice.

\section{Additional transport-ablation experiment: ridge regularisation robustness}
\label{app:transport_ablation}

To test whether the principal transport-dependent conclusions depend on one particular regularisation setting, we performed a transport-ablation experiment on the baseline model \texttt{meta-llama/Llama-3.2-3B-Instruct} at layer $\ell=16$. In this experiment we held the activations, cluster graph, local bases, and all downstream analysis settings fixed, and varied only the ridge parameter $\lambda$ used in the edge-transport estimation step. In addition to the baseline value $\lambda=10^{-2}$, we evaluated $\lambda\in\{10^{-3},10^{-1}\}$. This sweep therefore isolates the effect of transport regularisation while leaving the remainder of the atlas construction unchanged.

All three runs completed successfully and produced nondegenerate outputs for holonomy and shearing. Across the full sweep, the gauge sanity check remained valid, the graph-level cycle structure was unchanged, and the principal transport-sensitive quantities varied only negligibly. In practice, the measured holonomy and shearing summaries are almost identical across the three regularisation settings. This indicates that the empirical conclusions are not an artefact of a finely tuned ridge parameter.

\subsection{Experimental configuration}
\label{app:transport_ablation_setup}

Table~\ref{tab:transport_ablation_setup} summarises the fixed configuration used for the transport-ablation experiment. The only quantity varied across runs is the ridge parameter $\lambda$ used in the transport estimation stage.

\begin{table}[h]
\centering
\small
\begin{tabular}{ll}
\hline
\textbf{Setting} & \textbf{Value} \\
\hline
Model & \texttt{meta-llama/Llama-3.2-3B-Instruct} \\
Layer & $16$ \\
Activation source & fixed baseline activations \\
Graph source & fixed baseline cluster graph \\
Local bases source & fixed baseline local bases \\
Atlas size & $C=128$ \\
Local basis dimension & $k=32$ \\
Ridge parameters tested & $\lambda\in\{10^{-3},10^{-2},10^{-1}\}$ \\
Minimum overlap & $256$ \\
Maximum overlap & $8000$ \\
\hline
\end{tabular}
\caption{Configuration used for the transport-ablation experiment. All settings were fixed except the ridge parameter $\lambda$ used in the edge-transport estimation step.}
\label{tab:transport_ablation_setup}
\end{table}

\subsection{Summary of results}
\label{app:transport_ablation_table}

Table~\ref{tab:transport_ablation} summarises the principal holonomy and shearing quantities across the three ridge settings. The key observation is that the results are effectively unchanged across the full order-of-magnitude sweep on either side of the baseline value.

\begin{table*}[t]
\centering
\small
\begin{tabular}{lccc}
\hline
\textbf{Metric} & \textbf{$\lambda=10^{-3}$} & \textbf{$\lambda=10^{-2}$ (baseline)} & \textbf{$\lambda=10^{-1}$} \\
\hline
Saved edge ops & 192 & 192 & 192 \\
Connected components & 7 & 7 & 7 \\
Largest connected component size & 90 & 90 & 90 \\
Chord edges in LCC & 84 & 84 & 84 \\
Tree residual mean & $1.59\times 10^{-5}$ & $1.60\times 10^{-5}$ & $1.60\times 10^{-5}$ \\
Normalised holonomy mean $D_{\mathrm{hol}}$ & 0.554125 & 0.554122 & 0.554091 \\
Normalised holonomy max $D_{\mathrm{hol}}$ & 0.842083 & 0.842044 & 0.841658 \\
Usable edges in shearing analysis & 183 & 183 & 183 \\
Slack median & 8.118857 & 8.118773 & 8.117994 \\
Slack mean & 14.888190 & 14.886200 & 14.864558 \\
Slack max & 399.773431 & 399.525192 & 396.688072 \\
$D_{\mathrm{shear}}$ median & 0.205509 & 0.205509 & 0.205502 \\
$D_{\mathrm{shear}}$ mean & 0.198937 & 0.198935 & 0.198911 \\
$D_{\mathrm{shear}}$ max & 0.374262 & 0.374252 & 0.374152 \\
\hline
\end{tabular}
\caption{Transport-ablation sweep over the ridge parameter $\lambda$ on \texttt{meta-llama/Llama-3.2-3B-Instruct} at layer $16$. The baseline setting used in the main text is $\lambda=10^{-2}$.}
\label{tab:transport_ablation}
\end{table*}

\subsection{Holonomy robustness under transport regularisation}
\label{app:transport_ablation_holonomy}

As shown in Table~\ref{tab:transport_ablation}, the graph-level holonomy computation is essentially invariant across the tested ridge settings. The number of saved edge operators, the number of connected components, the size of the largest connected component, and the number of chord edges in that component remain exactly unchanged across $\lambda=10^{-3},10^{-2},10^{-1}$. The mean tree residual also remains at numerical precision throughout the sweep.

More importantly, the mean normalised holonomy $D_{\mathrm{hol}}$ is effectively identical across the three settings: it is $0.554125$ for $\lambda=10^{-3}$, $0.554122$ at baseline, and $0.554091$ for $\lambda=10^{-1}$. The maximum normalised holonomy changes only in the fourth decimal place. Thus, within this order-of-magnitude sweep around the baseline, the measured holonomy signal is not sensitive to the transport regularisation strength.

\subsection{Shearing robustness under transport regularisation}
\label{app:transport_ablation_shearing}

The shearing results are equally stable. Table~\ref{tab:transport_ablation} shows that the number of usable edges in the shearing analysis remains fixed at $183$ across the full sweep. The median shearing score $D_{\mathrm{shear}}$ is also essentially unchanged: $0.205509$ for $\lambda=10^{-3}$, $0.205509$ for the baseline setting, and $0.205502$ for $\lambda=10^{-1}$. The mean and maximum shearing scores are similarly stable.

The slack distribution likewise shows negligible variation. As reported in Table~\ref{tab:transport_ablation}, the median slack stays near $8.118$ throughout, while the mean and maximum slack change only slightly. This indicates that the tightness of the shearing lower bound is also robust to moderate changes in the ridge parameter. In practical terms, the typical transport mismatch diagnosed by the framework does not depend materially on whether the transport fit is regularised somewhat more weakly or somewhat more strongly than in the baseline run.

\subsection{Summary}
\label{app:transport_ablation_summary}

Taken together, these transport-ablation experiments show that the transport-dependent conclusions of the framework are robust to moderate variation in the ridge regularisation parameter. As summarised in Table~\ref{tab:transport_ablation}, an order-of-magnitude sweep on either side of the baseline value leaves the principal holonomy and shearing statistics essentially unchanged. The graph-level cycle structure is identical across runs, the holonomy summaries are numerically indistinguishable to practical precision, and the shearing summaries vary only negligibly. These results support the view that the reported geometric structure is not an artefact of one finely tuned transport regularisation choice.

\end{document}

%% file: math_commands.tex
\usepackage{amsmath,amsfonts,bm}

\def\eqref#1{equation~\ref{#1}}

\def\1{\bm{1}}

\DeclareMathAlphabet{\mathsfit}{\encodingdefault}{\sfdefault}{m}{sl}
\SetMathAlphabet{\mathsfit}{bold}{\encodingdefault}{\sfdefault}{bx}{n}

\newcommand{\E}{\mathbb{E}}

\newcommand{\R}{\mathbb{R}}

